\documentclass[twoside,11pt]{article}
\pdfoutput=1
\usepackage{jmlr2e}
\usepackage{amsmath}

\newtheorem{assumption}[theorem]{Assumption}
\usepackage{thm-restate}

\DeclareMathOperator*{\argmin}{argmin}

\def \bE {\mathbb{E}}

\def \bN {\mathbb{N}}

\def \bP {\mathbb{P}}

\def \bR {\mathbb{R}}

\def \cC {\mathcal{C}}

\def \cE {\mathcal{E}}
\def \cF {\mathcal{F}}
\def \cG {\mathcal{G}}
\def \cH {\mathcal{H}}

\def \cM {\mathcal{M}}
\def \cN {\mathcal{N}}
\def \cO {\mathcal{O}}
\def \cP {\mathcal{P}}

\def \cR {\mathcal{R}}
\def \cS {\mathcal{S}}
\def \cT {\mathcal{T}}

\def \cW {\mathcal{W}}
\def \cX {\mathcal{X}}

\def \Pdim {\,{\rm Pdim}\,}
\def \Lip {\,{\rm Lip}\,}
\def \Bin {\,{\rm Bin}\,}
\def \mid {\,{\rm mid}\,}
\def \sgn {\,{\rm sgn}\,}

\usepackage{lastpage}
\jmlrheading{23}{2022}{1-\pageref{LastPage}}{7/21; Revised 2/22}{4/22}{21-0732}{Jian Huang, Yuling Jiao, Zhen Li, Shiao Liu, Yang Wang and Yunfei Yang}
\ShortHeadings{Error Analysis of GANs}{Huang, Jiao, Li, Liu, Wang and Yang}

\begin{document}
\title{An Error Analysis of Generative Adversarial Networks for Learning Distributions}

\author{\name Jian Huang 
\email jian-huang@uiowa.edu \\
\addr Department of Statistics and Actuarial Science\\ University of Iowa, Iowa City, Iowa, USA
\AND
\name Yuling Jiao
\email yulingjiaomath@whu.edu.cn \\
\addr School of Mathematics and Statistics \\ and Hubei Key Laboratory of Computational Science \\ Wuhan University, Wuhan, China
\AND
\name Zhen Li
\email lishen03@gmail.com 
\AND
\name Shiao Liu 
\email shiao-liu@uiowa.edu \\
\addr Department of Statistics and Actuarial Science\\ University of Iowa, Iowa City, Iowa, USA
\AND
\name Yang Wang
\email yangwang@ust.hk \\
\addr Department of Mathematics\\ The Hong Kong University of Science and Technology
\\Clear Water Bay, Kowloon, Hong Kong, China
\AND 
\name Yunfei Yang \thanks{Corresponding author.}
\email yyangdc@connect.ust.hk \\
\addr Department of Mathematics\\ The Hong Kong University of Science and Technology
\\Clear Water Bay, Kowloon, Hong Kong, China
}

\editor{Mehryar Mohri}
\maketitle

\begin{abstract}
This paper studies how well generative adversarial networks (GANs) learn probability distributions from finite samples. Our main results establish the convergence rates of GANs under a collection of integral probability metrics defined through H\"older classes, including the Wasserstein distance as a special case. We also show that GANs are able to adaptively learn data distributions with low-dimensional structures or have H\"older densities, when the network architectures are chosen properly. In particular, for distributions concentrated around a low-dimensional set, we show that the learning rates of GANs do not depend on the high ambient dimension, but on the lower intrinsic dimension. Our analysis is based on a new oracle inequality decomposing the estimation error into the generator and discriminator approximation error and the statistical error, which may be of independent interest.
\end{abstract}

\begin{keywords}
Generative adversarial networks, deep neural networks, convergence rate, error decomposition, risk bound
\end{keywords}

\section{Introduction}

Generative adversarial networks (GANs, \citet{goodfellow2014generative,li2015generative,dziugaite2015training,arjovsky2017wasserstein}) have attracted much attention in machine learning and artificial intelligence communities in the past few years. As a powerful unsupervised method for learning and sampling from complex data distributions, GANs have achieved remarkable successes in many machine learning tasks such as image synthesis, medical imaging and natural language generation \citep{radford2016unsupervised,reed2016generative,zhu2017unpaired,
karras2018progressive,yi2019generative,bowman2016generating}. However, theoretical explanations for their empirical success are not well established. Many problems on the theory and training dynamics of GANs are largely unsolved \citep{arora2017generalization,liang2021how,singh2018nonparametric}.

Different from classical density estimation methods, GANs implicitly learn the data distribution by training a generator and a discriminator against each other. More specifically, to estimate a target distribution $\mu$, one chooses an easy-to-sample source distribution $\nu$ (for example, uniform or Gaussian distribution) and find the generator by solving the following minimax optimization problem, at the population level,
\[
\min_{g\in \cG} \max_{f\in \cF} \bE_{x\sim \mu} [f(x)] - \bE_{z\sim \nu} [f(g(z))],
\]
where both the generator class $\cG$ and the discriminator class $\cF$ are often parameterized by neural networks in general. The inner maximization problem can be viewed as that of calculating the Integral Probability Metric (IPM, see \citet{muller1997integral}) between the target $\mu$ and the generated distribution $g_\#\nu$ with respect to the discriminator class $\cF$:
\[
d_\cF(\mu,g_\#\nu) := \sup_{f\in \cF} \bE_\mu[f] - \bE_{g_\#\nu}[f] = \sup_{f\in \cF} \bE_{x\sim \mu} [f(x)] - \bE_{z\sim \nu} [f(g(z))],
\]
where $g_\#\nu$ is the push-forward distribution under $g$. When only a set of random samples $\{X_i\}_{i=1}^n$ that are independent and identically distributed (i.i.d.) as  $\mu$ are available in practice, we estimate the expectations by the empirical averages and  solve the empirical optimization problem
\begin{equation}\label{gan}
g_n^* = \argmin_{g \in \cG} d_\cF(\widehat{\mu}_n, g_\# \nu) = \argmin_{g \in \cG} \sup_{f\in \cF} \frac{1}{n} \sum_{i=1}^{n} f(X_i) - \bE_{z\sim \nu} [f(g(z))],
\end{equation}
where $\widehat{\mu}_n = \frac{1}{n} \sum_{i=1}^{n} \delta_{X_i}$ is the empirical distribution.

One of the fundamental questions in GANs is their generalization capacity: how well can GANs learn a target distribution from finite samples? Recently, much effort has been devoted to answering this question in different aspects. For example, \citet{arora2017generalization} showed that GANs do not generalize in standard metrics with any polynomial number of examples and provided generalization bounds for neural net distance. \citet{zhang2018discrimination} gave a detailed analysis of neural net distance and extended the results of \citet{arora2017generalization}. \citet{liang2021how} and \citet{singh2018nonparametric} analyzed the adversarial framework from a nonparametric density estimation point of view. \citet{chen2020statistical} studied the convergence properties of GANs when both the target densities and the evaluation class are H\"older classes.

While impressive progress has been made on the theoretical understanding of GANs, there are still some shortcomings in the existing results. For instance, the source and the target distributions are often assumed to have the same ambient dimension in the current theory, while, in practice, GANs are usually trained using a source distribution with ambient dimension much smaller than that of the target distribution. Indeed, an important strength of GANs is their ability to model latent structures of complex high-dimensional distributions using a low-dimensional source distribution. Another issue needs to be addressed is that the generalization bounds often suffer from the curse of dimensionality. In practical applications, the data distributions are of high dimensionality, which makes the convergence rates in theory extremely slow. However, high-dimensional data, such as images, texts and natural languages, often have latent low-dimensional structures, which reduces the complexity of the problem. It is desirable to take into account such structures in the analysis.

\subsection{Contributions}

In this paper, we provide an error analysis of GANs and establish their convergence rates in various settings. We show that, if the generator and discriminator network architectures are properly chosen, GANs are able to learn any distributions with bounded support. To be concrete, let $\mu$ be a probability distribution on $[0,1]^d$ and $g_n^*$ be a solution of the optimization problem (\ref{gan}), then $(g^*_n)_\# \nu$ is an estimate of $\mu$. Informally, our main result shows that the GAN estimator has the convergence rate
\[
\bE [d_{\cH^\beta}(\mu,(g^*_n)_\# \nu)] = \cO (n^{-\beta/d} \lor n^{-1/2}\log n),
\]
where the expectation
is with respect to the random samples. The performance of the estimator is evaluated by the IPM $d_{\cH^\beta}$ with respect to some H\"older class $\cH^\beta$ of smoothness index $\beta>0$. These metrics cover a wide range of popular metrics used in the literature, including the Wasserstein distance. In our theory, the ambient dimension of the source distribution $\nu$ is allowed to be different from the ambient dimension of the target distribution $\mu$. In particular, it can be much smaller than that of the target distribution, which is the case in practice. Moreover, the convergence rates we derived match the minimax optimal rates of nonparametric density estimation under adversarial losses \citep{liang2021how,singh2018nonparametric}.

We also adapt our error analysis to three cases: (1) the target distribution concentrates around a low-dimensional set, (2) the target distribution has a density function, and (3) the target distribution has an unbounded support. In particular, we prove that if the target $\mu$ is supported on a set with dimension $d^*$, then the GAN estimator $g_n^*$ has a faster convergence rate:
\[
\bE [d_{\cH^\beta}(\mu,(g^*_n)_\# \nu)] = \cO ((n^{-\beta/d^*} \lor n^{-1/2})\log n).
\]
This implies that the convergence rates of GANs do not depend on nominal high dimensionality of data, but on the lower intrinsic dimension. Our results show that GANs can automatically adapt to the support of the data and overcome the curse of dimensionality.

Our work also makes significant technical contributions to the error analysis of GANs and neural network approximation theory, which may be of independent interest. For example, we develop a new oracle inequality for GAN estimators, which decomposes the estimation error into generator and discriminator approximation error and statistical error. To bound the discriminator approximation error, we establish explicit error bounds on approximating H\"older functions by neural networks, with an explicit upper bound on the Lipschitz constant of the constructed neural network functions. To the best of our knowledge, this is the first approximation result that also controls the regularity of the neural network functions.

\subsection{Preliminaries and Notation}

Let us first introduce several definitions and notations. The set of positive integers is denoted by $\bN=\{1,2,\dots\}$. We also denote $\bN_0:= \bN \cup \{0\}$ for convenience. 
Let $A$ and $B$ be two quantities. The maximum and minimum of $A$ and $B$ are denoted by $A\lor B$ and $A \land B$ respectively. We use the asymptotic notation $A \precsim B$ and $B \succsim A$ to denote the statement that $A\le CB$ for some constant $C>0$. We denote $A\asymp B$ when $A \precsim B$ and $A \succsim B$.
Let $\nu$ be a measure on $\bR^k$ and $g:\bR^k \to \bR^d$ be a measurable mapping. The push-forward measure $g_\# \nu$ of a measurable set $A$ is defined as $g_\# \nu(A) := \nu(g^{-1}(A))$.

The ReLU function is denoted by $\sigma(x):= x\lor 0$. A neural network function $\phi:\bR^{N_0} \to \bR^{N_{L+1}}$ is a function that can be parameterized by a ReLU neural network in the following form
\[
\phi(x) = T_L(\sigma(T_{L-1}(\cdots \sigma(T_0(x))\cdots))),
\]
where the activation function $\sigma$ is applied component-wisely and $T_l(x) := A_l x +b_l$ is an affine transformation with $A_l \in \bR^{N_{l+1}\times N_{l}}$ and $b_l\in \bR^{N_{l+1}}$ for $l=0,\dots,L$. The numbers $W=\max\{N_1,\dots,N_L\}$ and $L$ are called the \emph{width} and the \emph{depth} of neural network, respectively. When the input and output dimensions are clear from contexts, we denote by $\cN\cN(W,L)$ the set of functions that can be represented by neural networks with width at most $W$ and depth at most $L$.

To measure the complexity of neural networks from a learning theory perspective, we use the following notion of combinatorial dimension for a real-valued function class.

\begin{definition}[Pseudo-dimension]
\textnormal{Let $\cH$ be a class of real-valued functions defined on $\Omega$. The \emph{pseudo-dimension} of $\cH$, denoted by $\Pdim(\cH)$, is the largest integer $N$ for which there exist points $x_1,\dots,x_N \in \Omega$ and constants $c_1,\dots,c_N\in \bR$ such that
$$
|\{ \sgn(h(x_1)-c_1),\dots,\sgn(h(x_N)-c_N): h\in \cH \}| =2^N.
$$
}
\end{definition}

Next, let us introduce the notion of regularity for a function. For a multi-index $\alpha =(\alpha_1,\dots,\alpha_d) \in \bN_0^d$, the monomial on $x=(x_1,\dots,x_d)$ is denoted by $x^\alpha := x_1^{\alpha_1} \cdots x_d^{\alpha_d}$, the $\alpha$-derivative of a function $h$ is denoted by $\partial^\alpha h := \partial^{\|\alpha\|_1} h/ \partial x_1^{\alpha_1}\cdots \partial x_d^{\alpha_d}$ with $\|\alpha\|_1 = \sum_{i=1}^d \alpha_i$ as the usual $1$-norm for vectors. We use the convention that $\partial^\alpha h:=h$ if $\|\alpha\|_1=0$.

\begin{definition}[Lipschitz functions]
\textnormal{Let $\cX\subseteq \bR^d$ and $h:\cX \to \bR$, the \emph{Lipschitz constant} of $h$ is denoted by
\[
\Lip h := \sup_{x,y\in \cX, x\neq y} \frac{|h(x)-h(y)|}{\|x-y\|_2}.
\]
We denote $\Lip(\cX,K)$ as the set of all functions $h:\cX \to \bR$ with $\Lip h\le K$. For any $B>0$, we denote $\Lip(\cX,K,B):= \{ h\in \Lip(\cX,K): \|h\|_{L^\infty(\cX)} \le B \}$.
}
\end{definition}

\begin{definition}[H\"older classes]
\textnormal{For $\beta>0$ with $\beta = s+r$, where $s\in \bN_0$ and $r\in (0,1]$, and $d\in \bN$, we denote the H\"older class $\cH^\beta(\bR^d)$ as
\[
\cH^\beta(\bR^d) := \left\{ h:\bR^d\to \bR, \max_{\|\alpha\|_1\le s} \|\partial^\alpha h\|_\infty \le 1, \max_{\|\alpha\|_1=s} \sup_{x\neq y} \frac{|\partial^\alpha h(x)- \partial^\alpha h(x)|}{\|x-y\|_2^r}\le 1 \right\}.
\]
For any subset $\cX\subseteq \bR^d$, we denote $\cH^\beta(\cX) := \{ h:\cX\to \bR, h\in \cH^\beta(\bR^d) \}$.
}
\end{definition}

It should be noticed that for $\beta=s+1$, we do \emph{not} assume that $h\in C^{s+1}$. Instead, we only require that $h\in C^s$ and its derivatives of order $s$ are Lipschitz continuous with respect to the metric $\|\cdot\|_2$. Note that, if $\beta\le 1$, then $|h(x) - h(y)| \le \|x-y\|_2^\beta$; if $\beta>1$, then $|h(x) - h(y)| \le \sqrt{d} \|x-y\|_2$. In particular, with the above definitions, $\cH^1([0,1]^d) = \Lip([0,1]^d,1,1)$. We will use the covering number to measure the complexity of a H\"older class.

\begin{definition}[Covering number]
\textnormal{Let $\rho$ be a pseudo-metric on $\cM$ and $S \subseteq\cM$. For any $\epsilon>0$, a set $A\subseteq \cM$ is called an \emph{$\epsilon$-covering} of $S$ if for any $x\in S$ there exists $y\in A$ such that $\rho(x,y)\le \epsilon$. The \emph{$\epsilon$-covering number} of $S$, denoted by $\cN(\epsilon,S,\rho)$, is the minimum cardinality of any $\epsilon$-covering of $S$.
}
\end{definition}

Finally, the composition of two functions $f:\bR^d \to \bR$ and $g:\bR^k \to \bR^d$ is denoted by $f\circ g(x) := f(g(x))$. We use $\cF \circ \cG := \{ f\circ g: f\in \cF,g\in \cG \}$ to denote the composition of two function classes. A function class $\cF$ is called \emph{symmetric} if $f\in \cF$ implies $-f\in \cF$.

\subsection{Outline}

The rest of the paper is organized as follows. Section \ref{sec2} presents our main result on the error analysis of GANs, where we assume that the target distribution has a compact support. In section \ref{sec3}, we extend the result to three different cases: (1) the target distribution is low-dimensional; (2) the target has a H\"older density; (3) the target has unbounded support. Section \ref{sec4} discusses related theoretical results of deep neural networks and GANs. Finally, Section \ref{sec5} gives the proofs of technical lemmas, including error decomposition and bounds on the approximation error and statistical error.

\section{Error Analysis of GANs}\label{sec2}

Let $\mu$ be an unknown target probability distribution on $\bR^d$, and let $\nu$ be a known and easy-to-sample distribution on $\bR^k$ such as uniform or Gaussian distribution. Suppose we have $n$ i.i.d. samples $\{X_i\}_{i=1}^n$ from $\mu$ and $m$ i.i.d. samples $\{Z_i\}_{i=1}^m$ from $\nu$.
Denote the corresponding empirical distributions by $\widehat{\mu}_n = \frac{1}{n} \sum_{i=1}^{n} \delta_{X_i}$ and $\widehat{\nu}_m = \frac{1}{m} \sum_{i=1}^{m} \delta_{Z_i},$  respectively. We consider the following two optimization problems
\begin{align}
\argmin_{g \in \cG} d_\cF(\widehat{\mu}_n, g_\# \nu) &= \argmin_{g \in \cG} \sup_{f\in \cF} \left\{ \frac{1}{n} \sum_{i=1}^{n} f(X_i) - \bE_{\nu} [f \circ g] \right\}, \label{GAN1} \\
\argmin_{g \in \cG} d_\cF(\widehat{\mu}_n, g_\# \widehat{\nu}_m) &= \argmin_{g \in \cG} \sup_{f\in \cF} \left\{ \frac{1}{n} \sum_{i=1}^{n} f(X_i) - \frac{1}{m} \sum_{j=1}^{m} f(g(Z_i)) \right\}, \label{GAN2}
\end{align}
where the generator class $\cG$ is parameterized by a ReLU neural network $\cN\cN(W_1,L_1)$ with width at most $W_1$ and depth at most $L_1$, and the discriminator class $\cF$ is parameterized by another ReLU neural network $\cN\cN(W_2,L_2)$.

\subsection{Convergence Rates of GAN Estimators}
We study the convergence rates of the GAN estimators $g^*_n$ and $g^*_{n,m}$ that solve the optimization problems (\ref{GAN1}) and (\ref{GAN2}) with optimization error $\epsilon_{opt}\ge 0$. In other words,
\begin{align}
g^*_n &\in \left\{g\in \cG: d_\cF(\widehat{\mu}_n, g_\# \nu) \le \inf_{\phi\in \cG} d_\cF(\widehat{\mu}_n, \phi_\# \nu) + \epsilon_{opt} \right\}, \label{gan estimator g_n} \\
g^*_{n,m} &\in \left\{g\in \cG: d_\cF(\widehat{\mu}_n, g_\# \widehat{\nu}_m) \le \inf_{\phi\in \cG} d_\cF(\widehat{\mu}_n, \phi_\# \widehat{\nu}_m) + \epsilon_{opt} \right\}. \label{gan estimator g_nm}
\end{align}
The performance is evaluated by the IPM
between the target $\mu$ and the learned distribution $\gamma = (g^*_n)_\# \nu$ or $\gamma = (g^*_{n,m})_\# \nu$ with respect to some function class $\cH$:
\[
d_{\cH}(\mu,\gamma) := \sup_{h\in \cH}  \bE_{x\sim \mu} [h(x)] - \bE_{y\sim \gamma} [h(y)].
\]
By specifying $\cH$ differently, one can obtain a list of commonly-used metrics:
\begin{itemize}
    \item when $\cH= \Lip(\bR^d,1)$ is the $1$-Lipschitz function class, then $d_\cH=\cW_1$ is the Wasserstein distance, which is used in the Wasserstein GAN \citep{arjovsky2017wasserstein};
    \item when $\cH= \Lip(\bR^d,B,B)$ is the bounded Lipschitz function class, then $d_\cH$ is the Dudley metric, which metricizes weak convergence \citep{dudley2018real};
    \item when $\cH$ is the set of continuous function, then $d_\cH$ is the total variation distance;
    \item when $\cH$ is a Sobolev function class with certain regularity, $d_\cH$ is used in Sobolev GAN \citep{mroueh2018sobolev};
    \item when $\cH$ is the unit ball of some reproducing kernel Hilbert space, then $d_\cH$ is the maximum mean discrepancy \citep{gretton2012kernel,dziugaite2015training, li2015generative}.
\end{itemize}
Here, we consider the case when $\cH$ is a H\"older class $\cH^\beta (\bR^d)$, which covers a wide range of applications. For simplicity, we first consider the case when $\mu$ is supported on the compact set $[0,1]^d$ and extend it to different situations in the next section. The main result is summarized in the following theorem.

\begin{theorem}\label{main theorem}
Suppose the target $\mu$ is supported on $[0,1]^d$, the source distribution $\nu$ is absolutely continuous on $\bR$ and the evaluation class is $\cH = \cH^\beta(\bR^d)$. Then, there exist a generator $\cG = \{g\in \cN\cN(W_1,L_1): g(\bR) \subseteq [0,1]^d \}$ with
\[
W_1^2L_1 \precsim n,
\]
and a discriminator $\cF = \cN\cN(W_2,L_2) \cap \Lip(\bR^d, K,1)$ with
\[
W_2L_2 \precsim n^{1/2} \log^2 n, \quad K \precsim (\widetilde{W}_2\widetilde{L}_2)^{2+\sigma(4\beta-4)/d} \widetilde{L}_2 2^{\widetilde{L}_2^2},
\]
where $\widetilde{W}_2 = W_2/\log_2 W_2$ and $\widetilde{L}_2 = L_2/\log_2 L_2$, such that the GAN estimator (\ref{gan estimator g_n}) satisfies
\[
\bE [d_\cH(\mu,(g^*_n)_\# \nu)] - \epsilon_{opt} \precsim n^{-\beta/d} \lor n^{-1/2}\log^{c(\beta,d)} n,
\]
where $c(\beta,d)=1$ if $2\beta = d$, and $c(\beta,d)=0$ otherwise.

If furthermore $m\succsim n^{2+2\beta/d}\log^6 n$, then the GAN estimator (\ref{gan estimator g_nm}) satisfies
\[
\bE [d_\cH(\mu,(g^*_{n,m})_\# \nu)] - \epsilon_{opt} \precsim n^{-\beta/d} \lor n^{-1/2}\log^{c(\beta,d)} n.
\]
\end{theorem}

Before proceeding, we make several remarks on the theorem.

\begin{remark}
\textnormal{If $\beta=1$, then $\cH^1([0,1]^d) = \Lip([0,1]^d, 1,1)$ and $d_{\cH^1}$ is the Dudley distance (the Wasserstein distance $\cW_1$ on $[0,1]^d$ is IPM with the class $\Lip([0,1]^d,1)$ or $\Lip([0,1]^d,1,\sqrt{d})$, and it satisfies $\cW_1(\mu,\gamma) \le \sqrt{d} d_{\cH^1}(\mu,\gamma)$). In this case, the required Lipschitz constant of the discriminator network is reduced to $K \precsim \widetilde{W}_2^2\widetilde{L}_2^3 2^{\widetilde{L}_2^2}$. If we choose the depth $L_2$ to be a constant, then the Lipschitz constant can be chosen to have the order of $K \precsim \widetilde{W}_2^2 \precsim n\log^2n$.
}
\end{remark}

\begin{remark}
\textnormal{For simplicity, we assume that the source distribution $\nu$ is on $\bR$. This is not a restriction, because any absolutely continuous distribution on $\bR^k$ can be projected to an absolutely continuous distribution on $\bR$ by linear mapping. Hence, the same result holds for any absolutely continuous source distribution on $\bR^k$.
The requirement on the generator that $g(\bR) \subseteq [0,1]^d$ is easy to satisfy by adding an additional clipping layer to the output and using the fact that
\[
\min\{ \max\{ x, -1\}, 1 \} = \sigma(x+1) - \sigma(x-1) -1, \quad x\in \bR.
\]
}
\end{remark}

\begin{remark}\label{regularized GAN}
\textnormal{The Lipschitz condition on the discriminator might be difficult to satisfy in practice. It is done by weight clipping in the original Wasserstein GAN \citep{arjovsky2017wasserstein}. In the follow-up works \citep{gulrajani2017improved,kodali2017convergence,petzka2018regularization,wei2018improving,thanhtung2019improving}, several regularization methods have been applied to Wasserstein GANs. It would be interesting to develop similar error analysis for regularized GAN estimators, and we leave this as future work.
}
\end{remark}

\subsection{Error Decomposition}

Our proof of Theorem \ref{main theorem} is based on a new error decomposition and estimation of approximation error and statistical error sketched below. The proofs of technical lemmas are deferred to Section \ref{sec5}.

We first introduce a new oracle inequality, which decomposes the estimation error into the generator approximation error, the discriminator approximation error and the statistical error.

\begin{lemma}\label{error decomposition}
Assume $\cF$ is symmetric, $\mu$ and $g_\#\nu$ are supported on $\Omega\subseteq \bR^d$ for all $g\in \cG$. Let $g^*_n$ and $g^*_{n,m}$ be the GAN estimators (\ref{gan estimator g_n}) and (\ref{gan estimator g_nm}) respectively. Then, for any function class $\cH$ defined on $\Omega$,
\begin{align*}
d_\cH(\mu,(g^*_n)_\# \nu) &\le \epsilon_{opt} + 2\cE(\cH,\cF,\Omega)  + \inf_{g \in \cG} d_\cF(\widehat{\mu}_n,g_\# \nu) + d_\cF(\mu,\widehat{\mu}_n) \land d_\cH(\mu,\widehat{\mu}_n), \\
d_\cH(\mu,(g^*_{n,m})_\# \nu) &\le \epsilon_{opt} + 2\cE(\cH,\cF,\Omega)  + \inf_{g \in \cG} d_\cF(\widehat{\mu}_n,g_\# \nu) + d_\cF(\mu,\widehat{\mu}_n) \land d_\cH(\mu,\widehat{\mu}_n) \\
& \quad + 2d_{\cF \circ \cG}(\nu,\widehat{\nu}_m),
\end{align*}
where $\cE(\cH,\cF,\Omega)$ is the approximation error of $\cH$ from $\cF$ on $\Omega$:
\[
\cE(\cH,\cF,\Omega) := \sup_{h\in \cH} \inf_{f\in \cF} \|h-f\|_{L^\infty(\Omega)}.
\]
\end{lemma}

Next, we bound each error term separately. We will show that the generator approximation error $\inf_{g \in \cG} d_\cF(\widehat{\mu}_n,g_\# \nu) =0$ as long as the size of the generator network $\cG$ is sufficiently large. The discriminator approximation error $\cE(\cH,\cF,\Omega)$ can be bounded by constructing neural networks to approximate functions in $\cH$. The remaining statistical error terms can be controlled using the empirical process theory.

\subsubsection{Bounding Generator Approximation Error}  Observe that the empirical distribution $\widehat{\mu}_n$ is supported on at most $n$ points. To bound the generator approximation error $\inf_{g \in \cG} d_\cF(\widehat{\mu}_n,g_\# \nu)$, we need to estimate the distance between the generated distribution $\{g_\#\nu :g\in \cG \}$ and the set of all discrete distribution supported on at most $n$ points:
\[
\cP(n):= \left\{\gamma= \sum _{i=1}^n p_i\delta_{x_i}: \sum_{i=1}^n p_i=1, p_i\ge 0, x_i\in \bR^d \right\}.
\]
\citet{yang2022capacity} showed that their Wasserstein distance vanishes when the generator class is sufficiently large.

\begin{lemma}\label{app discrete measure}
Suppose that $W\ge 7d+1$, $L\ge 2$ and $\cG=\cN\cN(W,L)$. Let $\nu$ be an absolutely continuous probability distribution on $\bR$. If $n\le \frac{W-d-1}{2} \lfloor \frac{W-d-1}{6d} \rfloor \lfloor \frac{L}{2} \rfloor +2$, then for any $\gamma\in \cP(n)$ and any $\epsilon>0$, there exists $g\in \cG$ such that
\[
\cW_1(\gamma,g_\#\nu) <\epsilon.
\]
If the support of $\gamma$ is contained in some convex set $\cC$, then $g$ can be chosen to satisfy $g(\bR) \subseteq \cC$.
\end{lemma}

Since $\widehat{\mu}_n$ is supported on $[0,1]^d$, if we choose the generator $\cG = \{g\in \cN\cN(W_1,L_1): g(\bR) \subseteq [0,1]^d \}$ that satisfies the condition in the Lemma \ref{app discrete measure}, which means we can choose $W_1^2L_1 \precsim n$, then for any $\cF \subseteq \Lip([0,1]^d,K)$, we have
\[
\inf_{g \in \cG} d_\cF(\widehat{\mu}_n,g_\# \nu) \le K \inf_{g \in \cG} \cW_1(\widehat{\mu}_n,g_\# \nu) =0.
\]
This shows that the generator approximation error vanishes.

\subsubsection{Bounding Discriminator Approximation Error} To bound $\cE(\cH,\cF,[0,1]^d)$, we construct a neural network to approximate any given function in $\cH^\beta([0,1]^d)$. Our construction is based on the idea in \citet{daubechies2021nonlinear,shen2020deep} and \citet{lu2021deep}. More importantly, we give an upper bound on the Lipschitz constant of the neural network function that achieves small approximation error.

\begin{restatable}{lemma_re}{disapp}\label{discriminator approximation}
Assume $h\in \cH^\beta([0,1]^d)$ with $\beta = s+r$, $s\in \bN_0$ and $r\in (0,1]$. For any $W\ge 6$, $L\ge 2$, there exists $\phi \in \cN\cN(49(s+1)^2 3^d d^{s+1}W \lceil\log_2 W\rceil, 15(s+1)^2 L \lceil \log_2 L\rceil+2d)$ such that $\|\phi\|_\infty \le 1$, $\Lip \phi \le (s+1) d^{s+1/2}L(WL)^{\sigma(4\beta-4)/d} (1260 W^2L^2 2^{L^2}+ 19s 7^s)$ and
\[
\| \phi - h\|_{L^\infty([0,1]^d)} \le 6(s+1)^2 d^{(s+\beta/2) \lor 1} \lfloor (WL)^{2/d}\rfloor^{-\beta}.
\]
\end{restatable}

This lemma implies that, for any $h\in \cH^\beta([0,1]^d)$, there exists a neural network $\phi$ with width $\precsim W\log_2 W$ and depth $\precsim L\log_2 L$ such that $\phi \in \Lip(\bR^d,K,1)$ with Lipschitz constant $K \precsim (WL)^{2+\sigma(4\beta-4)/d} L 2^{L^2}$ and $\| \phi - h\|_{L^\infty([0,1]^d)} \precsim (WL)^{-2\beta/d}$. Hence, if we choose $W_2 \asymp W\log_2 W$ and $L_2 \asymp L\log_2 L$, then
\[
W \asymp W_2/\log_2 W_2 = \widetilde{W}_2, \quad L \asymp L_2/ \log_2 L_2 =\widetilde{L}_2,
\]
and $\phi \in \cN\cN(W_2,L_2) \cap \Lip(\bR^d, K,1)$ with
\[
K \precsim (WL)^{2+\sigma(4\beta-4)/d} L 2^{L^2} \precsim (\widetilde{W}_2\widetilde{L}_2)^{2+\sigma(4\beta-4)/d} \widetilde{L}_2 2^{\widetilde{L}_2^2}.
\]
This shows that, for the discriminator $\cF = \cN\cN(W_2,L_2) \cap \Lip(\bR^d, K,1)$,
\[
\cE(\cH^\beta,\cF,[0,1]^d) \precsim (W_2L_2 / (\log_2 W_2 \log_2 L_2))^{-2\beta/d}.
\]

\subsubsection{Bounding Statistical Error}
 For any function class $\cF$, the statistical error $\bE [d_\cF(\mu,\widehat{\mu}_n)]$ can be bounded by the Rademacher complexity of $\cF$, by using the standard symmetrization technique. We can further bound the Rademacher complexity by the covering number of $\cF$. The result is summarized in the following lemma.

\begin{restatable}{lemma_re}{staterr}\label{statistical error bound}
Assume $\sup_{f\in \cF}\|f\|_\infty \le B$, then we have the following entropy integral bound
\[
\bE [d_\cF(\mu,\widehat{\mu}_n)] \le 8\bE_{X_{1:n}} \inf_{0< \delta<B/2}\left( \delta + \frac{3}{\sqrt{n}} \int_{\delta}^{B/2} \sqrt{\log \cN(\epsilon,\cF_{|_{X_{1:n}}},\|\cdot\|_\infty)} d\epsilon \right),
\]
where we denote $\cF_{|_{X_{1:n}}} = \{(f(X_1),\dots,f(X_n)):f\in \cF \}$ for any i.i.d. samples $X_{1:n}=\{X_i \}_{i=1}^n$ from $\mu$ and $\cN(\epsilon,\cF_{|_{X_{1:n}}},\|\cdot\|_\infty)$ is the $\epsilon$-covering number of $\cF_{|_{X_{1:n}}} \subseteq \bR^n$ with respect to the $\|\cdot\|_\infty$ distance.
\end{restatable}

For the H\"older class $\cH=\cH^\beta(\bR^d)$, for any i.i.d. samples $X_{1:n}=\{X_i \}_{i=1}^n$ from $\mu$, which is supported on $[0,1]^d$, we have
\[
\log \cN(\epsilon,\cH_{|_{X_{1:n}}},\|\cdot\|_\infty) \le \log \cN(\epsilon,\cH^\beta([0,1]^d),\|\cdot\|_\infty)  \precsim \epsilon^{-d/\beta},
\]
where the last inequality is from the entropy bound in \citet{kolmogorov1961} (see also Lemma \ref{holder covering number}). Thus, if we denote $\eta=d/(2\beta)$, then
\[
\bE [d_\cH(\mu,\widehat{\mu}_n)] \precsim \inf_{0< \delta<1/2}\left( \delta + n^{-1/2} \int_{\delta}^{1/2} \epsilon^{-\eta} d\epsilon \right).
\]
When $\eta<1$, one has
\[
\bE [d_\cH(\mu,\widehat{\mu}_n)] \precsim \inf_{0< \delta<1/2}\left( \delta + (1-\eta)^{-1}n^{-1/2} (2^{\eta-1} - \delta^{1-\eta}) \right) \precsim n^{-1/2}.
\]
When $\eta = 1$, one has
\[
\bE [d_\cH(\mu,\widehat{\mu}_n)] \precsim \inf_{0< \delta<1/2}\left( \delta + n^{-1/2} (-\log 2 - \log \delta) \right)\precsim n^{-1/2} \log n,
\]
where we take $\delta=n^{-1/2}$ in the last step. When $\eta>1$, one has
\[
\bE [d_\cH(\mu,\widehat{\mu}_n)] \precsim \inf_{0< \delta<1/2}\left( \delta + (\eta-1)^{-1}n^{-1/2} (\delta^{1-\eta} -2^{\eta-1}) \right) \precsim n^{-1/(2\eta)} = n^{-\beta/d},
\]
where we take $\delta=n^{-1/(2\eta)}$. Combining these cases together, we have
\begin{equation}\label{holder complexity}
\bE [d_\cH(\mu,\widehat{\mu}_n)] \precsim n^{-\beta/d} \lor n^{-1/2}\log^{c(\beta,d)} n,
\end{equation}
where $c(\beta,d)=1$ if $2\beta = d$, and $c(\beta,d)=0$ otherwise.

\subsection{Proof of Theorem \ref{main theorem}}
For the GAN estimator $g^*_n$, by Lemma \ref{error decomposition}, we have the error decomposition
\begin{equation}\label{decomp}
d_\cH(\mu,(g^*_n)_\# \nu) \le \epsilon_{opt} + 2\cE(\cH,\cF,[0,1]^d)  + \inf_{g \in \cG} d_\cF(\widehat{\mu}_n,g_\# \nu) + d_\cH(\mu,\widehat{\mu}_n).
\end{equation}
We choose the generator class $\cG$ with $W_1^2L_1 \precsim n$ that satisfies the condition in Lemma \ref{app discrete measure}. Then
\[
\inf_{g \in \cG} d_\cF(\widehat{\mu}_n,g_\# \nu)=0,
\]
since $\cF \subseteq \Lip([0,1]^d,K)$. By Lemma \ref{discriminator approximation}, for our choice of the discriminator class $\cF$,
\[
\cE(\cH,\cF,[0,1]^d) \precsim (W_2L_2 / (\log_2 W_2 \log_2 L_2))^{-2\beta/d} \precsim n^{-\beta/d},
\]
where we can choose $W_2L_2 \asymp n^{1/2} \log^2 n$ so that the last inequality holds. By Lemma \ref{statistical error bound},
\[
\bE [d_\cH(\mu,\widehat{\mu}_n)] \precsim n^{-\beta/d} \lor n^{-1/2}\log^{c(\beta,d)} n.
\]
In summary, by (\ref{decomp}), we have
\[
\bE [d_\cH(\mu,(g^*_n)_\# \nu)] - \epsilon_{opt} \precsim n^{-\beta/d} \lor n^{-1/2}\log^{c(\beta,d)} n.
\]

For the estimator $g^*_{n,m}$, we only need to estimate the extra term $\bE [d_{\cF \circ \cG}(\nu,\widehat{\nu}_m)]$ by Lemma \ref{error decomposition}. We can bound this statistical error by the entropy integral in Lemma \ref{statistical error bound}, and further bound it by the pseudo-dimension $\Pdim(\cF \circ \cG)$ of the network $\cF \circ \cG$ (see corollary \ref{statistical error bound by pdim}):
\[
\bE [d_{\cF \circ \cG}(\nu,\widehat{\nu}_m)] \precsim  \sqrt{\frac{ \Pdim(\cF \circ \cG) \log m}{m}}.
\]
It was shown in \citet{bartlett2019nearly} that the pseudo-dimension of a ReLU neural network satisfies the bound $\Pdim(\cN\cN(W,L)) \precsim UL \log U$, where $U\asymp W^2L$ is the number of parameters. Hence,
\[
\bE [d_{\cF \circ \cG}(\nu,\widehat{\nu}_m)] \precsim \sqrt{\frac{(W_1^2L_1+W_2^2L_2)(L_1+L_2)\log(W_1^2L_1+W_2^2L_2) \log m}{m}}.
\]
Since we have chosen $W_2L_2 \precsim n^{1/2} \log^2 n$ and $W_1^2L_1 \precsim n$,  we have
\begin{align*}
\bE [d_{\cF \circ \cG}(\nu,\widehat{\nu}_m)] &\precsim \sqrt{\frac{(n+n\log^4n)(n+n^{1/2} \log^2 n) \log n \log m}{m}} \\
&\precsim \sqrt{\frac{n^2 \log^5 n \log m}{m}}.
\end{align*}
Hence, if $m\succsim n^{2+2\beta/d}\log^6 n$, then $\bE [d_{\cF \circ \cG}(\nu,\widehat{\nu}_m)] \precsim n^{-\beta/d}$ and, by Lemma \ref{error decomposition},
\[
\bE [d_\cH(\mu,(g^*_{n,m})_\# \nu)] - \epsilon_{opt} \precsim n^{-\beta/d} \lor n^{-1/2}\log^{c(\beta,d)} n,
\]
which completes the proof.

We make three remarks on the proof and the technical lemmas.

\begin{remark}\label{proof essential remark}
\textnormal{Our error decomposition for GANs in Lemma \ref{error decomposition} is different from the classical bias-variance decomposition for regression in the sense that the statistical error $d_\cF(\mu,\widehat{\mu}_n) \land d_\cH(\mu,\widehat{\mu}_n) \le d_\cH(\mu,\widehat{\mu}_n)$ depends on the evaluation class $\cH$. The proof of Theorem \ref{main theorem} essentially shows that we can choose the generator class and the discriminator class sufficiently large to reduce the approximation error so that the learning rate of GAN estimator is not slower than that of the empirical distribution.
}
\end{remark}

\begin{remark}
\textnormal{We give explicit estimate of the Lipschitz constant of the discriminator in Lemma \ref{discriminator approximation}, because it is essential in bounding the generator approximation error in our analysis. Alternatively, one can also bound the parameters in the discriminator network and then estimate the Lipschitz constant. For example, by using the construction in \citet{yarotsky2017error}, one can bound the weights as $\cO(\epsilon^{-\alpha})$ for some $\alpha>0$, where $\epsilon$ is the approximation error. Then convergence rates can be obtained for the discriminator network with bounded weights (the bound depends on the sample size $n$).
}
\end{remark}

\begin{remark}
\textnormal{The bound on the expectation $\bE [d_\cH(\mu,(g^*_n)_\# \nu)]$ can be turned into a high probability bound by using concentration inequalities \citep{boucheron2013concentration,shalevshwartz2014understanding,mohri2018foundations}. For example, by McDiarmid's inequality, one can shows that, for all $t> 0$,
\begin{equation}\label{probability bound}
\bP \left( d_\cH(\mu,\widehat{\mu}_n) \ge \bE [d_\cH(\mu,\widehat{\mu}_n)] + t \right) \le  \exp(-nt^2 /2),
\end{equation}
because for any $\{X_i\}_{i=1}^n$ and $\{X_i'\}_{i=1}^n$ that satisfies $X_i'=X_i$ except for $i=j$, we have
\[
\left|\sup_{h\in \cH} \left( \bE_\mu[h] - \frac{1}{n}\sum_{i=1}^n h(X_i) \right) - \sup_{h\in \cH} \left( \bE_\mu[h] - \frac{1}{n}\sum_{i=1}^n h(X_i') \right) \right| \le \sup_{h\in \cH} \frac{1}{n} \left| h(X_j) - h(X_j') \right| \le \frac{2}{n}.
\]
Since other error terms in inequality (\ref{decomp}) can be bounded independent of the random samples, it holds with probability at least $1-\delta$ that
\[
d_\cH(\mu,(g^*_n)_\# \nu) - \epsilon_{opt} - \sqrt{\frac{2\log(1/\delta)}{n}} \precsim n^{-\beta/d} \lor n^{-1/2}\log^{c(\beta,d)} n,
\]
where we choose $\exp(-nt^2 /2) =\delta$ in inequality (\ref{probability bound}).
}
\end{remark}

\section{Extensions of the Main Theorem}\label{sec3}

In this section, we extend the main theorem to the following cases: (1) the target distribution concentrates around a low-dimensional set, (2) the target distribution has a density function and, (3) the target distribution has an unbounded support.

\subsection{Learning Low-dimensional Distributions}

The convergence rates in Theorem \ref{main theorem} suffer from the curse of dimensionality. In practice, the ambient dimension is usually large, which makes the convergence very slow. However, in many applications, high-dimensional complex data such as images, texts and natural languages, tend to be supported on approximate lower-dimensional manifolds. To take into account this fact, we assume that the target distribution $\mu$ has a low-dimensional structure. We introduce the Minkowski dimension (or box-counting dimension) to determine the dimensionality of a set.

\begin{definition}[Minkowski dimension]\label{Minkowski dimensions}
\textnormal{The upper and the lower Minkowski dimensions of a set $A \subseteq \bR^d$ are defined respectively as
\begin{align*}
\overline{\dim}_M(A) := \limsup_{\epsilon\to 0} \frac{\log \cN(\epsilon,A,\|\cdot\|_2)}{-\log \epsilon}, \\
\underline{\dim}_M(A) := \liminf_{\epsilon\to 0} \frac{\log \cN(\epsilon,A,\|\cdot\|_2)}{-\log \epsilon}.
\end{align*}
If $\overline{\dim}_M(A) = \underline{\dim}_M(A) = \dim_M(A)$, then $\dim_M(A)$ is called the \emph{Minkowski dimension} of the set $A$.
}
\end{definition}

The Minkowski dimension measures how the covering number of $A$ decays when the radius of covering balls converges to zero. When $A$ is a manifold, its Minkowski dimension is the same as the dimension of the manifold. Since the Minkowski dimension only depends on the metric, it can also be used to measure the dimensionality of highly non-regular set, such as fractals \citep{falconer2004fractal}. For function classes defined on a set with a small Minkowski dimension, it is intuitive to expect that the covering number only depends on the intrinsic Minkowski dimension, rather than the ambient dimension. \citet{kolmogorov1961} gave a comprehensive study on such problems. We will need the following useful lemma in our analysis.

\begin{lemma}[\citet{kolmogorov1961}]\label{holder covering number}
If $\cX \subseteq \bR^d$ is a compact set with $\dim_M(\cX)=d^*$, then
\[
\log \cN(\epsilon,\cH^\beta(\cX), \|\cdot\|_\infty) \precsim \epsilon^{-d^*/\beta} \log(1/\epsilon).
\]
If, in addition, $\cX$ is connected, then
\[
\log \cN(\epsilon,\cH^\beta(\cX), \|\cdot\|_\infty) \precsim \epsilon^{-d^*/\beta}.
\]
\end{lemma}

For regression, \citet{nakada2020adaptive} showed that deep neural networks can adapt to the low-dimensional structure of data, and the convergence rates do not depend on the nominal high dimensionality of data, but on its lower intrinsic dimension. We will show that similar results hold for GANs by analyzing the learning rates of a target distribution that concentrates on a low-dimensional set.

\begin{assumption}\label{low-dim assumption}
\textnormal{The target $X\sim \mu$ has the form $X=\widetilde{X} + \xi$, where $\widetilde{X}$ and $\xi$ are independent, $\widetilde{X} \sim \widetilde{\mu}$ is supported on some compact set $\cX\subseteq [0,1]^d$ with $\dim_M(\cX)=d^*$, and $\xi$ has zero mean $\bE[\xi] =0$ and bounded variance $V=\bE[\|\xi\|_2^2 ] <\infty$.
}
\end{assumption}

The next theorem shows that the convergence rates of the GAN estimators only depend on the intrinsic dimension $d^*$, when the network architectures are properly chosen.

\begin{theorem}
Suppose the target $\mu$ satisfies assumption \ref{low-dim assumption}, the source distribution $\nu$ is absolutely continuous on $\bR$ and the evaluation class is $\cH = \cH^\beta(\bR^d)$. Then, there exist a generator $\cG = \{g\in \cN\cN(W_1,L_1): g(\bR) \subseteq [0,1]^d \}$ with
\[
W_1^2L_1 \precsim n,
\]
and a discriminator $\cF = \cN\cN(W_2,L_2) \cap \Lip(\bR^d, K,1)$ with
\[
W_2L_2 \precsim n^{d/(2d^*)} \log^2 n, \quad K \precsim (\widetilde{W}_2\widetilde{L}_2)^{2+\sigma(4\beta-4)/d} \widetilde{L}_2 2^{\widetilde{L}_2^2},
\]
where $\widetilde{W}_2 = W_2/\log_2 W_2$ and $\widetilde{L}_2 = L_2/\log_2 L_2$, such that the GAN estimator (\ref{gan estimator g_n}) satisfies
\[
\bE [d_\cH(\mu,(g^*_n)_\# \nu)] - \epsilon_{opt} - 2\sqrt{d} V^{(\beta \land 1)/2} \precsim (n^{-\beta/d^*}  \lor n^{-1/2}) \log n.
\]

If furthermore
\[
m\succsim
\begin{cases}
n^{(3d+4\beta)/(2d^*)}\log^6 n \quad & d^*\le d/2,\\
n^{1+(d+2\beta)/d^*}\log^4 n \quad & d^*> d/2,
\end{cases}
\]
then the GAN estimator (\ref{gan estimator g_nm}) satisfies
\[
\bE [d_\cH(\mu,(g^*_{n,m})_\# \nu)] - \epsilon_{opt} - 2\sqrt{d} V^{(\beta \land 1)/2} \precsim (n^{-\beta/d^*}  \lor n^{-1/2}) \log n.
\]
\end{theorem}
\begin{proof}
For any i.i.d. observations ${X_{1:n}}=\{X_i \}_{i=1}^n$ from $\mu$, where $X_i=\widetilde{X}_i + \xi_i$ with $\widetilde{X}_i\sim \widetilde{\mu}$, we denote $\widehat{\mu}_n = \frac{1}{n} \sum_{i=1}^{n} \delta_{X_i}$ and $\widehat {\widetilde{\mu}}_n = \frac{1}{n} \sum_{i=1}^{n} \delta_{\widetilde{X}_i}$. As in the proof of Theorem \ref{main theorem}, by Lemma \ref{general error decomposition}, we have
\begin{align*}
\bE [d_\cH(\mu,(g^*_n)_\# \nu)] &\le d_\cH(\mu,\widetilde{\mu}) + \bE [d_\cH(\widetilde{\mu},(g^*_n)_\# \nu)] \\
&\le d_\cH(\mu,\widetilde{\mu}) + 2\cE(\cH,\cF,[0,1]^d) +  \bE [d_\cH(\widetilde{\mu},\widehat{\mu}_n)] + \epsilon_{opt},
\end{align*}
and there exists a discriminator $\cF$ with $W_2L_2 \asymp n^{d/(2d^*)} \log^2 n$ such that
\[
\cE(\cH,\cF,[0,1]^d) \precsim (W_2L_2 / (\log_2 W_2 \log_2 L_2))^{-2\beta/d} \precsim n^{-\beta/d^*}.
\]

For the term $d_\cH(\mu,\widetilde{\mu})$, we can bound it as
\begin{equation}\label{low-dim ineq}
d_\cH(\mu,\widetilde{\mu}) = \sup_{h\in \cH} \bE_{\xi}[\bE_{\widetilde{X}} [h(\widetilde{X}+ \xi) - h(\widetilde{X})] ] \le \sqrt{d} \bE_{\xi}[\|\xi\|_2^{\beta \land 1}] \le \sqrt{d} V^{(\beta \land 1)/2},
\end{equation}
where we use the Lipschitz inequality $|h(\widetilde{X}+ \xi) - h(\widetilde{X})| \le \sqrt{d} \|\xi\|_2^{\beta \land 1}$ for the second inequality, and Jensen's inequality for the last inequality.

For the statistical error, we have
\[
\bE_{X_{1:n}} [d_\cH(\widetilde{\mu},\widehat{\mu}_n)] \le \bE_{\widetilde{X}_{1:n}} d_\cH(\widetilde{\mu}, \widehat {\widetilde{\mu}}_n) + \bE_{\xi_{1:n}} \bE_{\widetilde{X}_{1:n}} d_\cH(\widehat {\widetilde{\mu}}_n, \widehat{\mu}_n).
\]
Using Lipschitz continuity of $h$, we have
\begin{align}
\bE_{\xi_{1:n}} \bE_{\widetilde{X}_{1:n}} d_\cH(\widehat {\widetilde{\mu}}_n, \widehat{\mu}_n) &= \bE_{\xi_{1:n}} \bE_{\widetilde{X}_{1:n}} \sup_{h\in \cH} \frac{1}{n} \sum_{i=1}^n h(\widetilde{X}_i+ \xi_i) - h(\widetilde{X}_i) \notag \\
&\le \sqrt{d} \bE_{\xi_{1:n}} \frac{1}{n} \sum_{i=1}^n \|\xi_i\|_2^{\beta \land 1} \label{lip ineq} \\
&\le \sqrt{d} V^{(\beta \land 1)/2}. \notag
\end{align}
To estimate $\bE_{\widetilde{X}_{1:n}} d_\cH(\widetilde{\mu}, \widehat {\widetilde{\mu}}_n)$, recall that we have denoted $\cH_{|_{\widetilde{X}_{1:n}}} := \{(h(\widetilde{X}_1),\dots,h(\widetilde{X}_n)):h\in \cH \} \subseteq \bR^n$. Since $\widetilde{\mu}$ is supported on $\cX$ with $\dim_M(\cX)=d^*$ by Assumption \ref{low-dim assumption}, the covering number of $\cH_{|_{\widetilde{X}_{1:n}}}$ with respect to the distance $\|\cdot\|_\infty$ on $\bR^n$ can be bounded by the covering number of $\cH$ with respect to the $L^\infty(\cX)$ distance. Hence,
\[
\log \cN(\epsilon,\cH_{|_{\widetilde{X}_{1:n}}},\|\cdot\|_\infty) \le \log \cN(\epsilon,\cH^\beta(\cX), \|\cdot\|_\infty) \precsim \epsilon^{-d^*/\beta}\log(1/\epsilon),
\]
by Lemma \ref{holder covering number}. Therefore, by Lemma \ref{statistical error bound},
\begin{align*}
\bE_{\widetilde{X}_{1:n}} d_\cH(\widetilde{\mu}, \widehat {\widetilde{\mu}}_n) &\le 8\bE_{\widetilde{X}_{1:n}} \inf_{0< \delta<1/2}\left( \delta + \frac{3}{\sqrt{n}} \int_{\delta}^{1/2} \sqrt{\log \cN(\epsilon,\cH_{|_{\widetilde{X}_{1:n}}},\|\cdot\|_\infty)} d\epsilon \right) \\
&\precsim \inf_{0< \delta<1/2}\left( \delta + n^{-1/2} \int_{\delta}^{1/2} \epsilon^{-d^*/(2\beta)} \log (1/\epsilon) d\epsilon \right) \\
&\precsim \inf_{0< \delta<1/2}\left( \delta + n^{-1/2} \log (1/\delta) \int_{\delta}^{1/2} \epsilon^{-d^*/(2\beta)} d\epsilon \right).
\end{align*}
A calculation similar to the inequality (\ref{holder complexity}) gives
\[
\bE_{\widetilde{X}_{1:n}} d_\cH(\widetilde{\mu}, \widehat {\widetilde{\mu}}_n) \precsim (n^{-\beta/d^*}  \lor n^{-1/2}) \log n.
\]
Therefore,
\[
\bE_{X_{1:n}} [d_\cH(\widetilde{\mu},\widehat{\mu}_n)] - \sqrt{d} V^{(\beta \land 1)/2} \precsim (n^{-\beta/d^*}  \lor n^{-1/2}) \log n.
\]

In summary, we obtain the desired bound
\[
\bE [d_\cH(\mu,(g^*_n)_\# \nu)] - \epsilon_{opt} - 2\sqrt{d} V^{(\beta \land 1)/2} \precsim (n^{-\beta/d^*}  \lor n^{-1/2}) \log n.
\]

For the estimator $g^*_{n,m}$, we use the pseudo-dimension to bound $\bE [d_{\cF \circ \cG}(\nu,\widehat{\nu}_m)]$. Since we have chosen $W_2L_2 \precsim n^{d/(2d^*)} \log^2 n$ and $W_1^2L_1 \precsim n$,
\begin{align*}
\bE [d_{\cF \circ \cG}(\nu,\widehat{\nu}_m)] &\precsim \sqrt{\frac{(W_1^2L_1+W_2^2L_2)(L_1+L_2)\log(W_1^2L_1+W_2^2L_2) \log m}{m}} \\
&\precsim \sqrt{\frac{(n+n^{d/d^*}\log^4n)(n+n^{d/(2d^*)} \log^2 n) \log n \log m}{m}} \\
&\precsim \sqrt{\frac{n^{d/d^*}(n+n^{d/(2d^*)} \log^2 n) \log^5 n \log m}{m}}.
\end{align*}
By our choice of $m$, we always have $\bE [d_{\cF \circ \cG}(\nu,\widehat{\nu}_m)] \precsim n^{-\beta/d^*} \log n$. The result then follows from Lemma \ref{error decomposition}.
\end{proof}
\begin{remark}
\textnormal{In the proof, we actually show that the same convergence rate holds for $\widetilde{\mu}$: $\bE [d_\cH(\widetilde{\mu},(g^*_n)_\# \nu)] - \epsilon_{opt} - \sqrt{d} V^{(\beta \land 1)/2} \precsim (n^{-\beta/d^*}  \lor n^{-1/2}) \log n$. Note that the constant $\sqrt{d}$ is due to the Lipschitz constant of the evaluation class $\cH^\beta$. When $\beta=1$, we have a better Lipschitz inequality $|h(\widetilde{X}+ \xi) - h(\widetilde{X})| \le \|\xi\|_2$ in inequalities (\ref{low-dim ineq}) and (\ref{lip ineq}). As a consequence, one can check that, for the Dudley metric,
\[
\bE [d_{\cH^1}(\mu,(g^*_n)_\# \nu)] - \epsilon_{opt} - 2 V^{1/2} \precsim (n^{-1/d^*}  \lor n^{-1/2}) \log n.
\]
This bound is useful only when the variance term $V^{1/2}$ is negligible, i.e. the data distribution is really low-dimensional. One can regard the variance as a ``measure'' of how well the low-dimension assumption is fulfilled. It is numerically confirmed that several well-known real data have small intrinsic dimensions, while their nominal dimensions are very large \citep{nakada2020adaptive}.
}
\end{remark}

\subsection{Learning Distributions with Densities}

When the target distribution $\mu$ has a density function $p_\mu\in \cH^\alpha([0,1]^d)$, it was proved in \citet{liang2021how,singh2018nonparametric} that the minimax convergence rates of nonparametric density estimation satisfy
\[
\inf_{\widetilde{\mu}_n} \sup_{p_\mu \in \cH^\alpha([0,1]^d)} \bE d_{\cH^\beta([0,1]^d)}(\mu, \widetilde{\mu}_n) \asymp n^{-(\alpha+\beta)/(2\alpha +d)} \lor n^{-1/2},
\]
where the infimum is taken over all estimator $\widetilde{\mu}_n$ with density $p_{\widetilde{\mu}_n} \in \cH^\alpha([0,1]^d)$ based on $n$ i.i.d. samples $\{X_i\}_{i=1}^n$ of $\mu$. Ignoring the logarithmic factor, Theorem \ref{main theorem} gives the same convergence rate with $\alpha=0$, which reveals the optimality of the result (since we do not assume the target has density in Theorem \ref{main theorem}).

Under a priori that $p_\mu\in \cH^\alpha$ for some $\alpha>0$, it is not possible for the GAN estimators (\ref{gan estimator g_n}) and (\ref{gan estimator g_nm}) to learn the regularity of the target, because the empirical distribution $\widehat{\mu}_n$ do not inherit the regularity. However, we can use certain regularized empirical distribution $\widetilde{\mu}_n$ as the plug-in for GANs and consider the estimators
\begin{align}
\widetilde{g}^*_n &\in \left\{g\in \cG: d_\cF(\widetilde{\mu}_n, g_\# \nu) \le \inf_{\phi\in \cG} d_\cF(\widetilde{\mu}_n, \phi_\# \nu) + \epsilon_{opt} \right\}, \label{regularized gan estimator g_n} \\
\widetilde{g}^*_{n,m} &\in \left\{g\in \cG: d_\cF(\widetilde{\mu}_n, g_\# \widehat{\nu}_m) \le \inf_{\phi\in \cG} d_\cF(\widetilde{\mu}_n, \phi_\# \widehat{\nu}_m) + \epsilon_{opt} \right\}. \label{regularized gan estimator g_nm}
\end{align}
By choosing the regularized distribution $\widetilde{\mu}_n$, the generator $\cG$ and the discriminator $\cF$ properly, we show that $\widetilde{g}^*_n$ and $\widetilde{g}^*_{n,m}$ can achieve faster convergence rates than the GAN estimators (\ref{gan estimator g_n}) and (\ref{gan estimator g_nm}), which use the empirical distribution $\widehat{\mu}_n$ as the plug-in. The result can be seen as a complement to the nonparametric results in \citep[Theorem 3]{liang2021how}.

\begin{theorem}\label{learning density}
Suppose the target $\mu$ has a density function $p_\mu \in \cH^\alpha([0,1]^d)$ for some $\alpha>0$, the source distribution $\nu$ is absolutely continuous on $\bR$ and the evaluation class is $\cH = \cH^\beta(\bR^d)$. Then, there exist a regularized empirical distribution $\widetilde{\mu}_n$ with density $p_{\widetilde{\mu}_n} \in \cH^\alpha([0,1]^d)$, a generator $\cG = \{g\in \cN\cN(W_1,L_1): g(\bR) \subseteq [0,1]^d \}$ with
\[
W_1^2L_1 \precsim n^{\frac{\alpha+\beta}{2\alpha+d} \frac{d+\beta +\sigma(2\beta-2)}{\beta}d},
\]
and a discriminator $\cF = \cN\cN(W_2,L_2) \cap \Lip(\bR^d, K,1)$ with
\[
W_2/ \log_2 W_2 \precsim n^{\frac{\alpha+\beta}{2\alpha+d} \frac{d}{2\beta}}, \quad L_2 \asymp 1, \quad K \precsim (W_2/\log_2 W_2)^{2+\sigma(4\beta-4)/d} \precsim n^{\frac{\alpha+\beta}{2\alpha+d} \frac{d+ \sigma(2\beta -2)}{\beta}},
\]
such that the GAN estimator (\ref{regularized gan estimator g_n}) satisfies
\[
\bE [d_\cH(\mu,(\widetilde{g}^*_n)_\# \nu)] - \epsilon_{opt} \precsim n^{-(\alpha+\beta)/(2\alpha +d)} \lor n^{-1/2}.
\]

If furthermore $m \succsim n^{\frac{2\alpha+2\beta}{2\alpha+d} (\frac{d +\beta +\sigma(2\beta-2)}{\beta}d+1)} \log^2 n$, then the GAN estimator (\ref{regularized gan estimator g_nm}) satisfies
\[
\bE [d_\cH(\mu,(\widetilde{g}^*_{n,m})_\# \nu)] - \epsilon_{opt} \precsim n^{-(\alpha+\beta)/(2\alpha +d)} \lor n^{-1/2}.
\]
\end{theorem}
\begin{proof}
\citet{liang2021how} and \citet{singh2018nonparametric} showed the existence of regularized empirical distribution $\widetilde{\mu}_n$ with density $p_{\widetilde{\mu}_n} \in \cH^\alpha([0,1]^d)$ that satisfies
\[
\bE d_{\cH}(\mu, \widetilde{\mu}_n) \precsim n^{-(\alpha+\beta)/(2\alpha +d)} \lor n^{-1/2}.
\]
Similar to Lemma \ref{error decomposition}, we can decompose the error as (see Lemma \ref{general error decomposition})
\[
d_\cH(\mu,(\widetilde{g}^*_n)_\# \nu) \le \epsilon_{opt} + 2\cE(\cH,\cF,[0,1]^d)  + \inf_{g \in \cG} d_\cF(\widetilde{\mu}_n,g_\# \nu) + d_\cH(\mu,\widetilde{\mu}_n).
\]
By Lemma \ref{discriminator approximation}, we can choose a discriminator $\cF$ that satisfies the condition in the theorem such that the discriminator approximation error can be bounded by
\[
\cE(\cH,\cF,[0,1]^d) \precsim (W_2L_2 / (\log W_2 \log L_2))^{-2\beta/d} \precsim n^{-(\alpha+\beta)/(2\alpha +d)}.
\]
For the generator approximation error, since $\cF \subseteq \Lip([0,1]^d,K)$,
\[
\inf_{g \in \cG} d_\cF(\widetilde{\mu}_n,g_\# \nu) \le K \inf_{g \in \cG} \cW_1(\widetilde{\mu}_n,g_\# \nu).
\]
It was shown in \citet{yang2022capacity} that (see also Corollary \ref{app bounded measure})
\[
\inf_{g \in \cG} \cW_1(\widetilde{\mu}_n,g_\# \nu) \precsim (W_1^2L_1)^{-1/d}.
\]
Hence, there exists a generator $\cG$ with $W_1^2L_1 \asymp n^{\frac{\alpha+\beta}{2\alpha+d} \frac{d+\beta +\sigma(2\beta-2)}{\beta}d}$ such that
\[
\inf_{g \in \cG} d_\cF(\widetilde{\mu}_n,g_\# \nu) \precsim K (W_1^2L_1)^{-1/d} \precsim n^{-(\alpha+\beta)/(2\alpha +d)}.
\]
In summary, we have
\[
\bE d_\cH(\mu,(g^*_n)_\# \nu) - \epsilon_{opt} \precsim n^{-(\alpha+\beta)/(2\alpha +d)} \lor n^{-1/2}.
\]

For the estimator $\widetilde{g}^*_{n,m}$, we only need to further bound $d_{\cF \circ \cG}(\nu,\widehat{\nu}_m)$ due to Lemma \ref{general error decomposition}. By corollary \ref{statistical error bound by pdim}, we can bound it using the pseudo-dimension of $\cF \circ \cG$:
\begin{align*}
\bE [d_{\cF \circ \cG}(\nu,\widehat{\nu}_m)] &\precsim \sqrt{\frac{(W_1^2L_1+W_2^2L_2)(L_1+L_2)\log(W_1^2L_1+W_2^2L_2) \log m}{m}} \\
&\precsim \sqrt{\frac{(n^{\frac{\alpha+\beta}{2\alpha+d} \frac{d+\beta +\sigma(2\beta-2)}{\beta}d} + n^{\frac{\alpha+\beta}{2\alpha+d} \frac{d}{\beta}} \log n) n^{\frac{\alpha+\beta}{2\alpha+d} \frac{d+\beta +\sigma(2\beta-2)}{\beta}d} \log n \log m}{m}} \\
&\precsim n^{\frac{\alpha+\beta}{2\alpha+d} \frac{d+\beta +\sigma(2\beta-2)}{\beta}d} \sqrt{\frac{\log n \log m}{m}}.
\end{align*}
Since $m \succsim n^{\frac{2\alpha+2\beta}{2\alpha+d} (\frac{d +\beta +\sigma(2\beta-2)}{\beta}d+1)} \log^2 n$, we have $\bE [d_{\cF \circ \cG}(\nu,\widehat{\nu}_m)] \precsim n^{-(\alpha+\beta)/(2\alpha +d)}$, which completes the proof.
\end{proof}

As we noted in Remark \ref{proof essential remark}, the proof essentially shows that the convergence rates of $\widetilde{g}^*_n$ and $\widetilde{g}^*_{n,m}$ are not worse than the convergence rate of $\bE d_{\cH}(\mu, \widetilde{\mu}_n)$ if we choose the network architectures properly.

\subsection{Learning Distributions with Unbounded Supports}

So far, we have assumed that the target distribution has a compact support. In this section, we show how to generalize the results to target distributions with unbounded supports. For simplicity, we only consider the case when the target $\mu$ is sub-exponential in the sense that
\begin{equation}\label{tail condition}
\mu(\{x \in \bR^d: \|x\|_\infty > \log t \}) \precsim t^{-a/d},
\end{equation}
for some $a>0$. The basic idea is to truncate the target distribution and apply the error analysis to the truncated distribution.

\begin{theorem}\label{learning unbounded}
Suppose the target $\mu$ satisfies condition (\ref{tail condition}), the source distribution $\nu$ is absolutely continuous on $\bR$ and the evaluation class is $\cH = \cH^\beta(\bR^d)$. Then, there exist a generator $\cG = \{g\in \cN\cN(W_1,L_1): g(\bR) \subseteq [-\beta a^{-1} \log n, \beta a^{-1} \log n]^d \}$ with
\[
W_1^2L_1 \precsim n
\]
and a discriminator $\cF = \cN\cN(W_2,L_2) \cap \Lip(\bR^d, K,1)$ with
\[
W_2L_2 \precsim n^{1/2} \log^{2+d/2} n, \quad K \precsim (\widetilde{W}_2\widetilde{L}_2)^{2+\sigma(4\beta-4)/d} \widetilde{L}_2 2^{\widetilde{L}_2^2}(2\beta a^{-1}\log n)^{\beta-1},
\]
where $\widetilde{W}_2 = W_2/\log_2 W_2$ and $\widetilde{L}_2 = L_2/\log_2 L_2$, such that the GAN estimator (\ref{gan estimator g_n}) satisfies
\[
\bE [d_\cH(\mu,(g^*_n)_\# \nu)] - \epsilon_{opt} \precsim n^{-\beta/d} \lor n^{-1/2}\log^{c(\beta,d)} n,
\]
where $c(\beta,d)=1$ if $2\beta = d$, and $c(\beta,d)=0$ otherwise.

If furthermore $m\succsim n^{2+2\beta/d}\log^{6+d} n$, then the GAN estimator (\ref{gan estimator g_nm}) satisfies
\[
\bE [d_\cH(\mu,(g^*_{n,m})_\# \nu)] - \epsilon_{opt} \precsim n^{-\beta/d} \lor n^{-1/2}\log^{c(\beta,d)} n.
\]
\end{theorem}

\begin{proof}
Without loss of generality, we assume $a=1$ in (\ref{tail condition}). Denote $A_n= [-\beta \log n, \beta \log n]^d$, then $1-\mu(A_n) \precsim n^{-\beta/d}$ by (\ref{tail condition}). We define an operator $\cT_n : \cP(\bR^d) \to \cP(A_n)$ on the set $\cP(\bR^d)$ of all probability distributions on $\bR^d$ by
\[
\cT_n \gamma = \gamma|_{A_n} + (1-\gamma(A_n)) \delta_0, \quad \gamma \in \cP(\bR^d),
\]
where $\mu|_{A_n}$ is the restriction to $A_n$ and $\delta_0$ is the point measure on the zero vector. Since any function $h\in \cH$ is bounded $\|h\|_\infty \le 1$, we have
\begin{align*}
d_\cH(\mu, \cT_n \mu) &= \sup_{h\in \cH} \int_{\bR^d} h(x) d\mu(x) - \int_{\bR^d} h(x) d\cT_n\mu(x) \\
&= \sup_{h\in \cH} \int_{\bR^d \setminus A_n} h(x) d\mu(x) - (1-\mu(A_n)) h(0) \\
&\le 2 (1-\mu(A_n)) \precsim n^{-\beta/d}.
\end{align*}
As a consequence, by the triangle inequality,
\[
d_\cH(\mu,(g^*_n)_\# \nu) - d_\cH(\cT_n\mu,(g^*_n)_\# \nu) \le d_\cH(\mu, \cT_n \mu) \precsim n^{-\beta/d}.
\]
Since $\cT_n\mu$ and $g_\#\nu$ are supported on $A_n$ for all $g\in \cG$, by Lemma \ref{general error decomposition},
\[
d_\cH(\cT_n\mu,(g^*_n)_\# \nu) \le \epsilon_{opt} + 2\cE(\cH,\cF,A_n)  + \inf_{g \in \cG} d_\cF(\widehat{\mu}_n,g_\# \nu) + d_\cH(\cT_n\mu,\widehat{\mu}_n).
\]

For the discriminator approximation error, we need to approximate any function $h\in \cH^\beta(A_n)$. We can consider the function $\widetilde{h} \in \cH^\beta([0,1]^d)$ defined by
\[
\widetilde{h}(x) = \frac{1}{(2\beta \log n)^\beta} h(\beta \log n (2x-1)).
\]
By Lemma \ref{discriminator approximation}, there exists $\widetilde{\phi} \in \cN\cN(W_2,L_2-1) \cap \Lip(\bR^d, K/(2\beta \log n)^{\beta-1},1)$ such that $\| \widetilde{h} - \widetilde{\phi} \|_{L^\infty([0,1]^d)} \precsim (W_2L_2 / (\log W_2 \log L_2))^{-2\beta/d}$. Define
\begin{align*}
\phi_0 (x) &:= (2\beta \log n)^\beta \widetilde{\phi}\left( \tfrac{x}{2\beta \log n} + \tfrac{1}{2} \right), \\
\phi (x) &:= \min\{ \max\{\phi_0(x),-1\},1 \} = \sigma(\phi_0(x)+1) - \sigma(\phi_0(x)-1) -1,
\end{align*}
then $\phi \in \cN\cN(W_2,L_2) \cap \Lip(\bR^d, K,1)$ and
\[
\|h - \phi \|_{L^\infty(A_n)} \precsim (W_2L_2 / (\log W_2 \log L_2))^{-2\beta/d} \log^\beta n.
\]
This shows that, if we choose $W_2L_2 \asymp n^{1/2} \log^{2+d/2} n$,
\[
\cE(\cH,\cF,A_n) \precsim (W_2L_2 / (\log W_2 \log L_2))^{-2\beta/d} \log^\beta n \precsim n^{-\beta/d}.
\]
For the generator approximation error,
\[
\inf_{g \in \cG} d_\cF(\widehat{\mu}_n,g_\# \nu) \le d_\cF(\widehat{\mu}_n, \cT_n \widehat{\mu}_n) + \inf_{g \in \cG} d_\cF(\cT_n \widehat{\mu}_n,g_\# \nu).
\]
By Lemma \ref{app discrete measure}, we can choose a generator $\cG$ with $W_1^2L_1 \precsim n$ such that the last term vanishes. Since $\|f\|_\infty \le 1$ for any $f\in \cF$, we have
\[
\bE d_\cF(\widehat{\mu}_n, \cT_n \widehat{\mu}_n) \le \bE [2 \widehat{\mu}_n(\bR^d \setminus A_n)] = 2\bE \left[ \frac{1}{n} \sum_{i=1}^n 1_{\{X_i\notin A_n\}} \right] = 2 \mu(\bR^d \setminus A_n) \precsim n^{-\beta/d}.
\]
For the statistical error, by Lemma \ref{statistical error bound},
\[
\bE d_\cH(\cT_n\mu,\widehat{\mu}_n) \le d_\cH(\cT_n\mu,\mu) + \bE d_\cH(\mu,\widehat{\mu}_n) \precsim n^{-\beta/d} \lor n^{-1/2}\log^{c(\beta,d)} n.
\]

In summary, we have shown that
\[
\bE [d_\cH(\mu,(g^*_n)_\# \nu)] - \epsilon_{opt} \precsim n^{-\beta/d} \lor n^{-1/2}\log^{c(\beta,d)} n.
\]
The error bound for $g^*_{n,m}$ can be estimated in a similar way.  By Lemma \ref{general error decomposition}, we only need to further bound $\bE [d_{\cF \circ \cG}(\nu,\widehat{\nu}_m)]$, which can be done as in the proof of Theorem \ref{main theorem}.
\end{proof}
\begin{remark}
\textnormal{When $\beta=1$, $\cH^1 = \Lip(\bR^d,1,1)$, the metric $d_{\cH^1}$ is the Dudley metric. For the Wasserstein distance $\cW_1$, we let $A_n = [2a^{-1}\log n, 2a^{-1}\log n]^d$, then
\begin{align*}
\cW_1(\mu, \cT_n \mu) &= \sup_{\Lip h\le 1} \int_{\bR^d \setminus A_n} h(x)-h(0) d\mu(x) \le \int_{\bR^d \setminus A_n} \|x\|_2 d\mu(x) \\
&\le \sqrt{d} \bE[\|X\|_\infty 1_{\{X\notin A_n\}}] = \sqrt{d} \int_0^\infty \mu(\|X\|_\infty 1_{\{X\notin A_n\}} >t) dt \\
&\precsim \int_0^{2a^{-1}\log n} n^{-2/d} dt + \int_{2a^{-1}\log n}^\infty 2^{-at/d} dt \\
&\precsim n^{-2/d} \log n.
\end{align*}
If we choose the generator $\cG = \{g\in \cN\cN(W_1,L_1): g(\bR) \subseteq A_n \}$ and the discriminator $\cF = \cN\cN(W_2,L_2) \cap \Lip(\bR^d, K,2a^{-1} \sqrt{d}\log n)$ satisfying the conditions in Theorem \ref{learning unbounded} with $\beta =1$, one can show that
\[
\bE [\cW_1(\mu,(g^*_n)_\# \nu)] - \epsilon_{opt} \precsim n^{-1/d} \lor n^{-1/2}\log^{c(1,d)} n,
\]
where the same convergence rate holds for $\bE \cW_1(\mu,\widehat{\mu}_n)$ by \citet{fournier2015rate}. When $m$ is chosen properly, the same rate holds for the estimator $g^*_{n,m}$.
}
\end{remark}

\section{Discussion and Related Works}\label{sec4}

It is well-known that one-hidden-layer neural networks can approximate any continuous function on a compact set \citep{cybenko1989approximation,hornik1991approximation,pinkus1999approximation}. Recent breakthroughs of deep learning have motivated many studies on the approximation capacity of deep neural networks \citep{yarotsky2017error,yarotsky2018optimal,yarotsky2020phase,shen2019nonlinear,shen2020deep,lu2021deep,petersen2018optimal}. These works quantify the approximation error of deep ReLU networks in terms of the number of parameters or neurons. Our result on bounding discriminator approximation error uses ideas similar to those in these papers. An important feature of Lemma \ref{discriminator approximation} is that it gives an explicit bound on the Lipschitz constant required for approximating H\"older functions, which is new in the literature.

In contrast to the vast amount of studies on function approximation by neural networks, there are only a few papers estimating the generator approximation error \citep{lee2017ability,bailey2018size,perekrestenko2020constructive,lu2020universal,
chen2020statistical,yang2022capacity}. The existing studies often assume that the source distribution and the target distribution have the same ambient dimension \citep{lu2020universal,chen2020statistical} or the distributions have some special form \citep{lee2017ability,bailey2018size,perekrestenko2020constructive}. However, these assumptions are not satisfied in practical applications. Our analysis of generator approximation is based on \citet{yang2022capacity}, which has the minimal requirement on the source and the target distributions.

The generalization errors of GANs have been studied in several recent works. \citet{arora2017generalization} showed that, in general,  GANs do not generalize under the Wasserstein distance and the Jensen-Shannon divergence with any polynomial number of samples. Alternatively, they estimated the generalization bound under the ``neural net distance'', which is the IPM with respect to the discriminator network. \citet{zhang2018discrimination} improved the generalization bound in \citet{arora2017generalization} by explicitly quantifying the complexity of the discriminator network. However, these generalization theories make the assumption that the generator can approximate the data distribution well under the neural net distance, while the construction of such generator network is unknown. Also, the neural net distance is too weak that it can be small when two distributions are not very close \citep[corollary 3.2]{arora2017generalization}. In contrast, our results explicitly state the network architectures and provide convergence rates of GANs under the Wasserstein distance.

Similar to our results, \citet{bai2019approximability} showed that GANs are able to learn distributions in Wasserstein distance, if the
discriminator class has strong distinguishing power against the generator class. But their theory requires each layer of the neural network generator to be invertible, and hence the width of the generator has to be the same with the input dimension, which is not the usual practice in applications.  In contrast, we do not make any invertibility assumptions, and allow the discriminator and the generator networks to be wide. The work of \citet{chen2020statistical} is the most related to ours. They studied statistical properties of GANs and established convergence rate $\cO(n^{-\beta/(2\beta+d)}\log^2 n)$ for distributions with H\"older densities, when the evaluation class is another H\"older class $\cH^\beta$. Their estimation on generator approximation is based on the optimal transport theory, which requires that the input and the output dimensions of the generator to be the same. In this paper, we study the same problem as \citet{chen2020statistical} and improve the convergence rate to $\cO(n^{-\beta/d} \lor n^{-1/2} \log n)$ for general probability distributions without any restrictions on the input and the output dimensions of the generator. Furthermore,  our results circumvent the curse of dimensionality if the data distribution has a low-dimensional structure, and establish the convergence rate $\cO((n^{-\beta/d^*}\lor n^{-1/2})\log n)$ when the distribution concentrates around a set with Minkowski dimension $d^*$. The recent work of \citet{schreuder2021statistical} also consider learning low-dimensional distributions by GANs. However, in their setting, the data distribution is generated from some smooth function and their GAN estimators are defined by directly minimizing H\"older IPMs, rather than using a discriminator network. Hence, our results are more general and practical.

There is another line of work \citep{liang2021how,singh2018nonparametric,uppal2019nonparametric} concerning the non-parametric density estimation under IPMs. For example, \citet{liang2021how} and \citet{singh2018nonparametric} established the minimax optimal rate $\cO(n^{-(\alpha+\beta)/(2\alpha+d)} \lor n^{-1/2} )$ for learning a Sobolev class with smoothness index $\alpha>0$, when the evaluation class is another Sobolev class with smoothness $\beta$. \citet{uppal2019nonparametric} generalized the minimax rate to Besov IPMs, where both the target density and the evaluation classes  are Besov classes. Our main result matches this optimal rate with $\alpha=0$ without any assumption on the regularity of the data distribution. Theorem \ref{learning density} shows that GAN is able to achieve the optimal rate by using a suitable regularized empirical distribution.

As we noted in Remark \ref{regularized GAN}, the Lipschitz constraint on the discriminator network may be difficult to satisfy in practical applications. Several regularization techniques \citep{gulrajani2017improved,kodali2017convergence,petzka2018regularization,
wei2018improving,thanhtung2019improving} have been applied to GANs and shown to have good empirical performance. It is interesting to see how these regularization techniques affect the convergence rates of GANs. We leave this problem for the future studies.

Finally, we note that there is an optimization error term in our results of convergence rates. So, in order to estimate the full error of GANs used in practice, one also need to estimate the optimization error, which is still a very difficult problem at present. Fortunately, our error analysis is independent of the optimization, so it is possible to combine it with other analysis of optimization. In our main theorems, we give bounds on the network size so that GANs can achieve the optimal convergence rates of learning distributions. In practice, as the network size and sample size get larger, the training becomes more difficult and hence the optimization error may become larger. So there is a trade-off between the optimization error and the bounds derived in this paper. This trade-off can provide some guide on the choice of network size in practice.

\section{Proofs of Technical Lemmas} \label{sec5}

This section provides the proofs of technical lemmas used in the error analysis of GANs. We will first give a general error decomposition of the estimation error in Subsection \ref{sec: err decomp}, and then bound the generator approximation error in Subsection \ref{sec: gen err}, the discriminator approximation error in Subsection \ref{sec: dis err} and the statistical error in Subsection \ref{sec: stat err}.

\subsection{Error Decomposition}\label{sec: err decomp}

In this subsection, we prove the error decomposition Lemma \ref{error decomposition}. Before the proof, we introduce the following useful lemma, which states that for any two probability distributions, the difference in IPMs with respect to two distinct evaluation classes will not exceed two times the approximation error between the two evaluation classes.  Recall that, for any $\Omega\subseteq \bR^d$ and function classes $\cF$ and $\cH$ defined on $\Omega$, we denote
\[
\cE(\cH,\cF,\Omega) := \sup_{h\in \cH} \inf_{f\in \cF} \|h-f\|_{L^\infty(\Omega)}.
\]

\begin{lemma}\label{IPM comparision}
For any probability distributions $\mu$ and $\gamma$ supported on $\Omega\subseteq \bR^d$,
\[
d_\cH(\mu,\gamma) \le d_\cF(\mu,\gamma) + 2\cE(\cH,\cF,\Omega).
\]
\end{lemma}
\begin{proof}
For any $\epsilon>0$, there exists $h_\epsilon\in \cH$ such that
\[
d_\cH(\mu,\gamma) =  \sup_{h\in \cH} \{ \bE_\mu [h] - \bE_\gamma [h] \} \le \bE_\mu [h_\epsilon] - \bE_\gamma [h_\epsilon] +\epsilon.
\]
Choose $f_\epsilon\in \cF$ such that $\|h_\epsilon-f_\epsilon\|_{L^\infty (\Omega)} \le \inf_{f\in \cF} \|h_\epsilon-f\|_{L^\infty (\Omega)} +\epsilon$, then
\begin{align*}
d_\cH(\mu,\gamma) \le& \bE_\mu [h_\epsilon-f_\epsilon] - \bE_\gamma [h_\epsilon-f_\epsilon]  + \bE_\mu [f_\epsilon] - \bE_\gamma [f_\epsilon] +\epsilon \\
\le& 2\|h_\epsilon-f_\epsilon\|_{L^\infty (\Omega)} + \bE_\mu [f_\epsilon] - \bE_\gamma [f_\epsilon] +\epsilon \\
\le& 2\inf_{f\in \cF} \|h_\epsilon-f\|_{L^\infty (\Omega)} + 2\epsilon + d_\cF(\mu,\gamma) + \epsilon \\
\le& 2\cE(\cH,\cF,\Omega) + d_\cF(\mu,\gamma) + 3\epsilon,
\end{align*}
where we use the assumption that $\mu$ and $\gamma$ are supported on $\Omega$ in the second inequality, and use the definition of IPM $d_\cF$ in the third inequality. Letting $\epsilon \to 0$, we get the desired result.
\end{proof}

The next lemma gives an error decomposition of GAN estimators associated with an estimator $\widetilde{\mu}_n$ of the target distribution $\mu$. Lemma \ref{error decomposition} is a special case of this lemma with $\widetilde{\mu}_n = \widehat{\mu}_n$ being the empirical distribution. In the proof, we use two properties of IPM: the triangle inequality $d_\cF(\mu,\gamma) \le d_\cF(\mu,\tau) + d_\cF(\tau,\gamma)$ and, if $\cF$ is symmetric, then $d_\cF(\mu,\gamma) = d_\cF(\gamma,\mu)$. These properties can be proved easily using the definition.

\begin{lemma}\label{general error decomposition}
Assume $\cF$ is symmetric, $\mu$ and $g_\#\nu$ are supported on $\Omega\subseteq \bR^d$ for all $g\in \cG$. For any probability distribution $\widetilde{\mu}_n$ supported on $\Omega$, let $\widetilde{g}^*_n$ and $\widetilde{g}^*_{n,m}$ be the associated GAN estimators defined by
\begin{align*}
\widetilde{g}^*_n &\in \left\{g\in \cG: d_\cF(\widetilde{\mu}_n, g_\# \nu) \le \inf_{\phi\in \cG} d_\cF(\widetilde{\mu}_n, \phi_\# \nu) + \epsilon_{opt} \right\},\\
\widetilde{g}^*_{n,m} &\in \left\{g\in \cG: d_\cF(\widetilde{\mu}_n, g_\# \widehat{\nu}_m) \le \inf_{\phi\in \cG} d_\cF(\widetilde{\mu}_n, \phi_\# \widehat{\nu}_m) + \epsilon_{opt} \right\}.
\end{align*}
Then, for any function class $\cH$ defined on $\Omega$,
\begin{align*}
d_\cH(\mu,(\widetilde{g}^*_n)_\# \nu) &\le \epsilon_{opt} + 2\cE(\cH,\cF,\Omega)  + \inf_{g \in \cG} d_\cF(\widetilde{\mu}_n,g_\# \nu) + d_\cF(\mu,\widetilde{\mu}_n) \land d_\cH(\mu,\widetilde{\mu}_n), \\
d_\cH(\mu,(\widetilde{g}^*_{n,m})_\# \nu) &\le \epsilon_{opt} + 2\cE(\cH,\cF,\Omega)  + \inf_{g \in \cG} d_\cF(\widetilde{\mu}_n,g_\# \nu) + d_\cF(\mu,\widetilde{\mu}_n) \land d_\cH(\mu,\widetilde{\mu}_n) \\
& \quad + 2d_{\cF \circ \cG}(\nu,\widehat{\nu}_m).
\end{align*}
\end{lemma}
\begin{proof}
By lemma \ref{IPM comparision} and the triangle inequality, for any $g\in \cG$,
\begin{align*}
d_\cH(\mu,g_\# \nu) &\le 2\cE(\cH,\cF,\Omega) + d_\cF(\mu, g_\# \nu) \\
&\le 2\cE(\cH,\cF,\Omega) + d_\cF(\mu,\widetilde{\mu}_n) + d_\cF(\widetilde{\mu}_n,g_\# \nu).
\end{align*}
Alternatively, we can apply the triangle inequality first and then use lemma \ref{IPM comparision}:
\begin{align*}
d_\cH(\mu,g_\# \nu) &\le d_\cH(\mu,\widetilde{\mu}_n) + d_\cH(\widetilde{\mu}_n,g_\# \nu) \\
&\le d_\cH(\mu,\widetilde{\mu}_n) + d_\cF(\widetilde{\mu}_n,g_\# \nu) + 2\cE(\cH,\cF,\Omega).
\end{align*}
Combining these two bounds, we have
\begin{equation}\label{error decomposition inequality}
d_\cH(\mu,g_\# \nu) \le 2\cE(\cH,\cF,\Omega)  + d_\cF(\widetilde{\mu}_n,g_\# \nu) + d_\cF(\mu,\widetilde{\mu}_n) \land d_\cH(\mu,\widetilde{\mu}_n).
\end{equation}
Letting $g=\widetilde{g}^*_n$ and observing that $d_\cF(\widetilde{\mu}_n,(\widetilde{g}^*_n)_\# \nu) \le \inf_{g \in \cG} d_\cF(\widetilde{\mu}_n,g_\# \nu) + \epsilon_{opt}$, we get the bound for $d_\cH(\mu,(\widetilde{g}^*_n)_\# \nu)$.

For $\widetilde{g}^*_{n,m}$, we only need to bound
$ d_\cF(\widetilde{\mu}_n,(\widetilde{g}^*_{n,m})_\# \nu)$. By the triangle inequality,
\[
d_\cF(\widetilde{\mu}_n,(\widetilde{g}^*_{n,m})_\# \nu) \le d_\cF(\widetilde{\mu}_n,(\widetilde{g}^*_{n,m})_\# \widehat{\nu}_m) + d_\cF((\widetilde{g}^*_{n,m})_\# \widehat{\nu}_m,(\widetilde{g}^*_{n,m})_\# \nu).
\]
By the definition of IPM, the last term can be bounded as
\[
d_\cF((\widetilde{g}^*_{n,m})_\# \widehat{\nu}_m,(\widetilde{g}^*_{n,m})_\# \nu) \le d_{\cF \circ \cG}(\widehat{\nu}_m,\nu).
\]
By the definition of $\widetilde{g}^*_{n,m}$ and the triangle inequality, we have, for any $g\in \cG$,
\begin{align*}
d_\cF(\widetilde{\mu}_n,(\widetilde{g}^*_{n,m})_\# \widehat{\nu}_m) -\epsilon_{opt} &\le d_\cF(\widetilde{\mu}_n,g_\# \widehat{\nu}_m) \le  d_\cF(\widetilde{\mu}_n,g_\# \nu) + d_\cF(g_\# \nu, g_\# \widehat{\nu}_m) \\ &\le d_\cF(\widetilde{\mu}_n,g_\# \nu) + d_{\cF \circ \cG}(\nu,\widehat{\nu}_m).
\end{align*}
Taking infimum over all $g\in \cG$, we have
\[
d_\cF(\widetilde{\mu}_n,(\widetilde{g}^*_{n,m})_\# \widehat{\nu}_m) \le \epsilon_{opt} + \inf_{g \in \cG} d_\cF(\widetilde{\mu}_n,g_\# \nu) + d_{\cF \circ \cG}(\nu,\widehat{\nu}_m).
\]
Therefore,
\[
d_\cF(\widetilde{\mu}_n,(\widetilde{g}^*_{n,m})_\# \nu) \le \epsilon_{opt} + \inf_{g \in \cG} d_\cF(\widetilde{\mu}_n,g_\# \nu) + 2d_{\cF \circ \cG}(\nu,\widehat{\nu}_m).
\]
Combining this with the inequality (\ref{error decomposition inequality}), we get the bound for $d_\cH(\mu,(\widetilde{g}^*_{n,m})_\# \nu)$.
\end{proof}

\subsection{Bounding Generator Approximation Error} \label{sec: gen err}

For completeness, we sketch the proof of Lemma \ref{app discrete measure}, whose detailed proof can be found in \citet{yang2022capacity}. The proof is essentially based on the fact that ReLU neural networks can express any piece-wise linear functions. The following lemma is a quantified description of this fact.

\begin{lemma}[\citet{yang2022capacity}, Lemma 3.1]\label{linear interpolation Rd}
Suppose that $W\ge 7d+1$, $L\ge 2$ and $N\le (W-d-1)\lfloor \frac{W-d-1}{6d} \rfloor \lfloor\frac{L}{2}\rfloor$. For any $z_0<z_1<\dots<z_N<z_{N+1}$, let $\cS^d(z_0,\dots,z_{N+1})$ be the set of all continuous piece-wise linear functions $g:\bR \to \bR^d$ which have breakpoints only at $z_i$, $0\le i\le N+1$, and are constant on $(-\infty,z_0)$ and $(z_{N+1},\infty)$. Then $\cS^d(z_0,\dots,z_{N+1}) \subseteq \cN\cN(W,L)$.
\end{lemma}

This lemma essentially says that $N \precsim W^2L/d$ is sufficient for $\cS^d(z_0,\dots,z_{N+1}) \subseteq \cN\cN(W,L)$. One can also show that it is also a necessary condition. To see this, we denote the number of parameters in $\cN\cN(W,L)$ by $n(W,L)\precsim W^2L$, and consider the function $F: \bR^{n(W,L)}\to \bR^{d(N+2)}$ defined by $F(\theta) := (f_\theta(z_0),\dots, f_\theta(z_{N+1}))$, where $f_\theta\in\cN\cN(W,L)$ denote the neural network function parameterized by $\theta \in \bR^{n(W,L)}$. Since $F$ is a piece-wise multivariate polynomial of $\theta$, it is Lipschitz continuous on any compact sets, hence it does not increase the Hausdorff dimension \citep[Theorem 2.8]{evans2018measure}. If $\cS^d(z_0,\dots,z_{N+1}) \subseteq \cN\cN(W,L)$, then $F$ is surjective, which implies $d(N+2) \le n(W,L)$. Thus, $N\precsim n(W,L)/d \precsim W^2L/d$ is necessary for $\cS^d(z_0,\dots,z_{N+1}) \subseteq \cN\cN(W,L)$.

Now, we sketch the proof of Lemma \ref{app discrete measure}. For any $\gamma \in \cP(n)$, we can assume $\gamma = \sum_{i=1}^{n} p_i \delta_{x_i}$ with $\sum_{i=1}^{n} p_i=1$, $p_i>0$ and $x_i\in \bR^d$. For any absolutely continuous probability measure $\nu$ on $\bR$, we can choose $2n-2$ points
\[
z_{3/2} < z_2 < z_{5/2} < \dots < z_{n-1/2}< z_n
\]
such that $\nu((z_i,z_{i+1/2})) \approx p_i$ for $1\le i\le n$, where we set $z_1=-\infty$ and $z_{n+1/2} = \infty$ for convenient. Then, we can construct a continuous piece-wise linear function $g$ such that $g(z) = x_i$ for $z\in (z_i,z_{i+1/2})$ and $g$ is linear on $(z_{i+1/2},z_{i+1})$. For such a function $g$, $g_\# \nu$ is supported on a union of line segments that pass through all $x_i$, and $g_\#\nu(\{x_i\}) \approx p_i$ for all $1\le i\le n$. Since $g\in \cS^d(z_{3/2}, z_2,\dots,z_n)$ with $2n-2\le (W-d-1)\lfloor \frac{W-d-1}{6d} \rfloor \lfloor\frac{L}{2}\rfloor$ breakpoints, Lemma \ref{linear interpolation Rd} tells us that $g\in \cN\cN(W,L)$. Using this construction, one can show that for any given $\epsilon>0$, there exists $g\in \cN\cN(W,L)$ such that
\[
\cW_1(\gamma,g_\#\nu) <\epsilon.
\]
Furthermore, in our construction, $g(\bR) = \cup_{i=1}^{n-1} g([z_i,z_{i+1}])$ is a union of line segments with endpoints $x_i$ and $x_{i+1}$. Hence, $g(\bR)$ must be contained in the convex hull of $\{x_i:1\le i\le n\}$. Thus, if the support of $\gamma$ is in a convex set $\cC$, $g$ can be chosen to satisfy $g(\bR)\subseteq \cC$.

Using Lemma \ref{app discrete measure}, we can also bound the generator approximation error of a distribution with bounded support.

\begin{corollary}\label{app bounded measure}
Let $\nu$ be an absolutely continuous probability distribution on $\bR$. Assume that $\mu$ is a probability distribution on $[0,1]^d$. Then, for any $W\ge 7d+1$ and $L\ge 2$, for the generator $\cG = \{g\in \cN\cN(W,L): g(\bR) \subseteq [0,1]^d \}$, one has
\[
\inf_{g\in \cG} \cW_1(\mu, g_\# \nu) \le C_d (W^2L)^{-1/d},
\]
where $C_d$ is a constant depending only on $d$.
\end{corollary}
\begin{proof}
Given any $k\in \bN$, we denote $A_k := \{ (i_1,\dots,i_d)/k: i_j \in \bN, i_j \le k, j=1,2,\dots,d \}$, whose cardinality is $k^d$. It is easy to see that there exists a partition $[0,1]^d = \cup_{x_i \in A_k} Q_{x_i}$ such that for any $x_i \in A_k$ and $x\in Q_{x_i}$, $\|x-x_i\|_2 \le \sqrt{d}/k$. We consider the discrete distribution
\[
\gamma_k := \sum_{x_i \in A_k} \mu(Q_{x_i}) \delta_{x_i}.
\]
Then,
\begin{align*}
\cW_1(\mu,\gamma_k) &= \sup_{\Lip h\le 1} \int_{[0,1]^d} h(x) d\mu(x) - \sum_{x_i \in A_k} \mu(Q_{x_i}) h(x_i) \\
&= \sup_{\Lip h\le 1} \sum_{x_i \in A_k} \int_{Q_{x_i}} h(x) -h(x_i) d\mu(x) \\
&\le \sum_{x_i \in A_k} \mu(Q_{x_i}) \frac{\sqrt{d}}{k} = \frac{\sqrt{d}}{k}.
\end{align*}
For any $W\ge 7d+1$ and $L\ge 2$, we choose the largest $k\in \bN$ such that $k^d \le \frac{W-d-1}{2} \lfloor \frac{W-d-1}{6d} \rfloor \lfloor \frac{L}{2} \rfloor +2$, then by triangle inequality and Lemma \ref{app discrete measure},
\[
\inf_{g\in \cG} \cW_1(\mu, g_\# \nu) \le \cW_1(\mu,\gamma_k) + \inf_{g\in \cG} \cW_1(\gamma_k, g_\# \nu) \le \frac{\sqrt{d}}{k} \le C_d (W^2L)^{-1/d},
\]
for some constant $C_d$ depending only on $d$.
\end{proof}

\subsection{Bounding Discriminator Approximation Error} \label{sec: dis err}

This subsection considers the discriminator approximation error. Our goal is to construct a neural network to approximate a function $h\in \cH^\beta([0,1]^d)$ with $\beta = s+r\ge 1$, $s\in \bN_0$ and $r\in (0,1]$. The main idea is to approximate the Taylor expansion of $h$. By \citet[Lemma A.8]{petersen2018optimal}, for any $x, x_0\in [0,1]^d$,
\[
\left| h(x) - \sum_{\|\alpha\|_1 \le s} \frac{\partial^\alpha h(x_0)}{\alpha !} (x-x_0)^\alpha \right| \le d^s \|x-x_0\|_2^\beta.
\]
The approximation of the Taylor expansion can be divided into three parts:
\begin{itemize}
	\item Partition $[0,1]^d$ into small cubes $\cup_\theta Q_\theta$, and construct a network $\psi$ that approximately maps each $x\in Q_\theta$ to a fixed point $x_\theta \in Q_\theta$. Hence, $\psi$ approximately discretize $[0,1]^d$.
	
	\item For any $\alpha$, construct a network $\phi_\alpha$ that approximates the Taylor coefficient $x\in Q_\theta \mapsto \partial^\alpha h(x_\theta)$. Once $[0,1]^d$ is discretized, this approximation is reduced to a data fitting problem.
	
	\item Construct a network $P_\alpha(x)$ to approximate the monomial $x^\alpha$. In particular, we can construct a network $\phi_\times$ that approximates the product function.
\end{itemize}
Then our construction of neural network can be written in the form
\[
\phi(x) = \sum_{\|\alpha\|_1 \le s} \phi_\times \left( \frac{\phi_\alpha(x)}{\alpha !}, P_\alpha(x-\psi(x)) \right) .
\]
We collect the required preliminary results in next two subsections and give a proof of Lemma \ref{discriminator approximation} in Subsection \ref{subsec: proof of dis app}.

\subsubsection{Data Fitting}

Given any $N+2$ samples $\{(x_i,y_i) \in \bR^2:i=0,1,\dots,N+1\}$ with $x_0<x_1<\cdots<x_N<x_{N+1}$, there exists a unique piece-wise linear function $\phi$ that satisfies the following three condition
\begin{enumerate}
	\item $\phi(x_i)=y_i$ for $i=0,1,\dots,N+1$.
	
	\item $\phi$ is linear on each interval $[x_i,x_{i+1}]$, $i=0,1,\dots,N$
	
	\item $\phi(x) = y_0$ for $x\in (-\infty,x_0)$ and $\phi(x) = y_{N+1}$ for $x\in (x_{N+1},\infty)$.
\end{enumerate}
We say $\phi$ is the linear interpolation of the given samples. Note that for any $x\in \bR$,
\[
\min_{0\le i\le N+1} y_i \le \phi(x) \le \max_{0\le i\le N+1} y_i, \quad \mbox{and} \quad \Lip \phi \le \max_{0\le i\le N} \left| \frac{y_{i+1}-y_i}{x_{i+1}-x_i} \right|.
\]
The next lemma estimates the required size of network to interpolate the given samples. Note that this lemma is a special case of Lemma \ref{linear interpolation Rd}, which is from \citet[Lemma 3.1]{yang2022capacity} and \citet[lemma 3.4]{daubechies2021nonlinear}.

\begin{lemma}\label{linear interpolation}
For any $W \ge 6$, $L\in \bN$ and any samples $\{(x_i,y_i)\in \bR^2:i=0,1,\dots,N+1\}$ with $x_0<x_1<\cdots<x_N<x_{N+1}$, where $N\le \lfloor W/6 \rfloor WL$, the linear interpolation of these samples $\phi \in \cN\cN(W+2,2L)$.
\end{lemma}

As an application of Lemma \ref{linear interpolation}, we show how to use a ReLU neural network to approximately discretize the input space $[0,1]^d$.

\begin{proposition}\label{partition map}
For any integers $W\ge 6$, $L\ge 2$, $d\ge 1$ and $0<\delta \le \frac{1}{3K}$ with $K=\lfloor (WL)^{2/d}\rfloor$, there exists a one-dimensional ReLU network $\phi\in \cN\cN(4W+3,4L)$ such that $\phi(x)\in [0,1]$ for all $x\in \bR$, $\Lip \phi \le \frac{2L}{K^2\delta^2}$ and
\[
\phi(x) = \tfrac{k}{K}, \quad \mbox{if } x\in \left[\tfrac{k}{K}, \tfrac{k+1}{K}-\delta\cdot 1_{\{k<K-1\}} \right], k=0,1,\dots,K-1.
\]
\end{proposition}

\begin{proof}
The proof is divided into two cases: $d=1$ and $d\ge 2$.

Case 1: $d=1$. We have $K= W^2L^2$ and denote $M=W^2L$. Then we consider the sample set
\[
\{ (\tfrac{m}{M},m): m=0,1,\dots,M-1 \} \cup \{(\tfrac{m+1}{M}-\delta, m):m=0,1,\dots,M-2 \} \cup \{(1,M-1)\}.
\]
Its cardinality is $2M=2W^2L\le \lfloor 4W/6 \rfloor (4W)L+2$. By Lemma \ref{linear interpolation}, the linear interpolation of these samples $\phi_1\in \cN\cN(4W+2,2L)$. In particular, $\phi_1(x)\in [0,M-1]$ for all $x\in \bR$, $\Lip \phi_1=1/\delta$ and
\[
\phi_1(x) = m, \quad \mbox{if } x\in \left[\tfrac{m}{M}, \tfrac{m+1}{M}-\delta\cdot 1_{\{m<M-1\}} \right], m=0,1,\dots,M-1.
\]

Next, we consider the sample set
\[
\{ (\tfrac{l}{ML},l): l=0,1,\dots,L-1 \} \cup \{(\tfrac{l+1}{ML}-\delta, l):l=0,1,\dots,L-2 \} \cup \{(\tfrac{1}{M},L-1)\}.
\]
Its cardinality is $2L$. By Lemma \ref{linear interpolation}, the linear interpolation of these samples $\phi_2\in \cN\cN(8,2L)$. In particular, $\phi_2(x) \in [0,L-1]$ for all $x\in \bR$, $\Lip\phi_2=1/\delta$ and for $m=0,1,\dots,M-1$, $l=0,1,\dots,L-1$, we have
\[
\phi_2 \left(x-\tfrac{1}{M} \phi_1(x) \right) = \phi_2 \left(x-\tfrac{m}{M} \right) = l, \quad \mbox{if } x\in \left[\tfrac{mL+l}{ML}, \tfrac{mL+l+1}{ML}-\delta\cdot 1_{\{mL+l<ML-1\}} \right].
\]

Define $\phi(x):= \frac{1}{M} \phi_1(x) + \frac{1}{ML} \phi_2 \left(\sigma(x)-\frac{1}{M} \phi_1(x) \right) \in [0,1]$. Then, it is easy to see that $\phi \in \cN\cN(4W+3,4L)$. For each $x\in \left[\frac{k}{K}, \frac{k+1}{K}-\delta\cdot 1_{\{k<K-1\}} \right]$ with $k\in \{0,1,\dots,K-1\}=\{0,1,\dots,ML-1\}$, there exists a unique representation $k=mL+l$ for $m\in \{0,1,\dots,M-1\}$, $l\in \{0,1,\dots,L-1\}$, and we have
\[
\phi(x)= \tfrac{1}{M} \phi_1(x) + \tfrac{1}{ML} \phi_2 \left(\sigma(x)-\tfrac{1}{M} \phi_1(x) \right) = \tfrac{mL+l}{ML} = \tfrac{k}{K}.
\]
Observing that the Lipschitz constant of the function $x\mapsto \sigma(x)-\frac{1}{M}\phi_1(x)$ is $\frac{1}{M \delta}$, the Lipschitz constant of $\phi$ is at most $\frac{1}{M} \frac{1}{\delta} +\frac{1}{ML} \frac{1}{\delta} \frac{1}{M\delta} \le \frac{2L}{K^2\delta^2}$.

Case 2: $d\ge 2$. We consider the sample set
\[
\{ (\tfrac{k}{K}, \tfrac{k}{K}): k=0,1,\dots,K-1 \} \cup \{(\tfrac{k+1}{K}-\delta, \tfrac{k}{K}): k=0,1,\dots,K-1 \} \cup \{(1,\tfrac{K-1}{K}) \}.
\]
Its cardinality is $2K \le 2W^{2/d}L^{2/d}\le \lfloor 4W/6 \rfloor (4W)L+2$. By Lemma \ref{linear interpolation}, the linear interpolation of these samples $\phi\in \cN\cN(4W+2,2L)$. In particular, $\phi(x)\in [0,1]$ for all $x\in \bR$,
\[
\phi(x) = \tfrac{k}{K}, \quad \mbox{if } x\in \left[\tfrac{k}{K}, \tfrac{k+1}{K}-\delta\cdot 1_{\{k<K-1\}} \right], k=0,1,\dots,K-1,
\]
and the Lipschitz constant of $\phi$ is $\frac{1}{K\delta}\le \frac{2L}{K^2\delta^2}$.
\end{proof}

Lemma \ref{linear interpolation} shows that a network $\cN\cN(W,L)$ can exactly fit $N\asymp W^2L$ samples. We are going to show that it can approximately fit $N\asymp (W/\log_2 W)^2(L/\log_2 L)^2$ samples. The construction is based on the bit extraction technique \citep{bartlett1998almost,bartlett2019nearly}. The following lemma shows how to extract a specific bit using ReLU neural networks. For convenient, we denote the binary representation as
\[
\Bin 0.x_1x_2\dots x_L := \sum_{j=1}^L x_j 2^{-j} \in[0,1],
\]
where $x_j \in \{0,1\}$ for all $j=1,2,\dots,L$.

\begin{lemma}\label{bit extraction}
For any $L\in\bN$, there exists $\phi \in \cN\cN(8,2L)$ such that $\phi(x,l) = x_l$ for $x = \Bin 0.x_1 x_2\dots x_L$ with $x_j\in \{0,1\}$ and $l=1,2,\dots,L$. Furthermore, $|\phi(x,l)-\phi(x',l')|\le 2 \cdot 2^{L^2} |x-x'|+ L|l-l'|$ for any $x,x',l,l'\in \bR$.
\end{lemma}
\begin{proof}
For any $x = \Bin 0.x_1 x_2\dots x_L$, we define $\xi_j := \Bin 0.x_j x_{j+1}\dots x_L$ for $j=1,2,\dots,L$. Then $\xi_1 = x$ and $\xi_{j+1} = 2\xi_j-x_j = \sigma(2 \sigma(\xi_j) - x_j)$ for $j=1,2,\dots,L-1$. Let
\[
T(x):= \sigma(2^L x-2^{L-1}+1) - \sigma(2^L x-2^{L-1}) =
\begin{cases}
0  &x\le 1/2-2^{-L}, \\
\mbox{linear} & 1/2-2^{-L} <x <1/2, \\
1 & x\ge 1/2.
\end{cases}
\]
It is easy to check that $x_j = T(\xi_j)$.

Denote $\delta_j = 1$ if $j=0$ and $\delta_j =0$ if $j\neq 0$ is an integer. Observing that
\[
\delta_j = \sigma(j+1) + \sigma(j-1) -2\sigma(j),
\]
and $t_1t_2 = \sigma(t_1+t_2-1)$ for any $t_1,t_2\in \{0,1\}$, we have
\begin{equation}\label{x_l expression}
x_l = \sum_{j=1}^L \delta_{l-j} x_j = \sum_{j=1}^L \sigma\left( \sigma(l-j+1) + \sigma(l-j-1) -2\sigma(l-j) +x_j -1 \right).
\end{equation}
If we denote the partial sum $s_{l,j} = \sum_{i=1}^j \sigma( \sigma(l-i+1) + \sigma(l-i-1) -2\sigma(l-i) +x_i -1 )$, then $x_l=s_{l,L}$.

For any $t_1,t_2,t_3 \in \bR$, we define a function $\psi(t_1,t_2,t_3)=(y_1,y_2,y_3) \in \bR^3$ by
\begin{align*}
y_1 &:= \sigma(2 \sigma (t_1) - T(t_1) ), \\
y_2 &:= \sigma(t_2) + \sigma(\sigma(t_3)+\sigma(t_3-2)-2\sigma(t_3-1)+T(t_1)-1), \\
y_3 &:= \max\{t_3-1,-L\} = \sigma(t_3-1+L) -L.
\end{align*}
Then, it is easy to check that $\psi \in \cN\cN(8,2)$. Using the expressions (\ref{x_l expression}) we have derived for $x_l$, one has
\[
\psi(\xi_j,s_{l,j-1},l-j+1) = (\xi_{j+1},s_{l,j},l-j), \quad l,j=1,\dots,L,
\]
where $s_{l,0}:=0$ and $\xi_{L+1}:=0$. Hence, by composing $\psi$ $L$ times, we can construct a network $\phi = \psi \circ \cdots \circ \psi \in \cN\cN(8,2L)$ such that $\phi(x,l) = \psi \circ \cdots \circ \psi(x,0,l) =s_{l,L} = x_l$ for $l=1,2,\dots,L$, where we drop the first and the third outputs of $\psi$ in the last layer.

It remains to estimate the Lipschitz constant. For any $t_1,t_2,t_3,t_1',t_2',t_3' \in \bR$, suppose $(y_1,y_2,y_3)= \psi(t_1,t_2,t_3)$ and $(y_1',y_2',y_3')= \psi(t_1',t_2',t_3')$. Then $|y_1-y_1'| \le 2^L |t_1-t_1'|$, $|y_3-y_3'|\le |t_3-t_3'|$ and $|y_2-y_2'|\le |t_2-t_2'|+2^L|t_1-t_1'| + |t_3-t_3'|$. Therefore, by induction,
\begin{align*}
|\phi(x,l)-\phi(x',l')| &\le (2^L + 2^{2L} + 2^{3L} + \cdots + 2^{L^2})|x-x'| + L|l-l'| \\
&\le 2 \cdot 2^{L^2} |x-x'| + L|l-l'|,
\end{align*}
for any $x,x',l,l'\in \bR$.
\end{proof}

Using the bit extraction technique, the next lemma shows a network $\cN\cN(W,L)$ can exactly fit $N\asymp W^2L^2$ binary samples.

\begin{lemma}\label{binary fitting}
Given any $W \ge 6$, $L\ge 2$ and any $\theta_i \in \{0,1\}$ for $i=0,1,\dots,W^2L^2-1$, there exists $\phi\in \cN\cN(8W+4,4L)$ such that $\phi(i)=\theta_i$ for $i=0,1,\dots,W^2L^2-1$ and
$\Lip \phi \le 2 \cdot 2^{L^2} + L^2$.
\end{lemma}
\begin{proof}
Denote $M=W^2L$, then, for each  $i=0,1,\dots,W^2L^2-1$, there exists a unique representation $i=mL+l$ with $m=0,1,\dots,M-1$ and $l=0,1,\dots,L-1$. So we define $b_{m,l}:= \theta_i$, where $i=mL+l$. We further set $y_m := \Bin 0.b_{m,0} b_{m,1}\dots b_{m,L-1} \in [0,1]$ and $y_M=1$. By Lemma \ref{bit extraction}, there exists $\psi\in \cN\cN(8,2L)$ such that $\psi(y_m,l+1) = b_{m,l}$ for any $m=0,1,\dots,M-1$, and $l=0,1,\dots,L-1$.

We consider the sample set
\[
\{(mL,y_m): m=0,1,\dots,M \} \cup \{(mL-1,y_{m-1}):m=1,\dots,M \}.
\]
Its cardinality is $2M+1=2W^2 L+1\le \lfloor 4W/6 \rfloor (4W)L+2$. By Lemma \ref{linear interpolation}, the linear interpolation of these samples $\phi_1\in \cN\cN(4W+2,2L)$. In particular, $\Lip \phi_1 \le 1$ and $\phi_1(i) =y_m$, when $i=mL+l$, for $m=0,1,\dots,M-1$, and $l=0,1,\dots,L-1$.

Similarly, for the sample set
\[
\{(mL,0): m=0,1,\dots,M \} \cup \{(mL-1,L-1):m=1,\dots,M \},
\]
the linear interpolation of these samples $\phi_2\in \cN\cN(4W+2,2L)$. In particular, $\Lip \phi_2 =L-1$ and $\phi_2(i) =l$, when $i=mL+l$, for $m=0,1,\dots,M-1$, and $l=0,1,\dots,L-1$.

We define $\phi(x):= \psi(\phi_1(x),\phi_2(x)+1)$, then $\phi\in \cN\cN(8W+4,4L)$ and
\[
\phi(i)= \psi(\phi_1(i),\phi_2(i)+1)= \psi(y_m,l+1) = b_{m,l} = \theta_i
\]
for $i=mL+l$ with $m=0,1,\dots,M-1$, and $l=0,1,\dots,L-1$. By Lemma \ref{bit extraction}, we have
\[
|\phi(x)-\phi(x')|\le 2 \cdot 2^{L^2} |\phi_1(x)-\phi_1(x')|+ L|\phi_2(x)-\phi_2(x')| \le (2 \cdot 2^{L^2} + L^2)|x-x'|
\]
for any $x,x'\in \bR$.
\end{proof}

As an application of Lemma \ref{binary fitting}, we show that a network $\cN\cN(W,L)$ can approximately fit $N\asymp (W/\log_2 W)^2(L/\log_2 L)^2$ samples.

\begin{proposition}\label{data fitting}
For any $W \ge 6$, $L\ge 2$, $s\in \bN$ and any $\xi_i \in [0,1]$ for $i=0,1,\dots,W^2L^2-1$, there exists $\phi\in \cN\cN(8s(2W+1) \lceil \log_2 (2W)\rceil +2,4L \lceil \log_2(2L)\rceil+1)$ such that $\Lip \phi \le 4 \cdot 2^{L^2} + 2L^2$, $|\phi(i)-\xi_i|\le (WL)^{-2s}$ for $i=0,1,\dots,W^2L^2-1$ and $\phi(t)\in [0,1]$ for all $t\in \bR$.
\end{proposition}
\begin{proof}
Denote $J= \lceil 2s\log_2(WL) \rceil$. For each $\xi_i\in [0,1]$, there exist $b_{i,1}, b_{i,2},\dots,b_{i,J} \in \{0,1\}$ such that
\[
|\xi_i - \Bin 0.b_{i,1}b_{i,2} \dots b_{i,J}| \le 2^{-J}.
\]
By Lemma \ref{binary fitting}, there exist $\phi_1,\phi_2,\dots,\phi_J \in \cN\cN(8W+4,4L)$ such that $\Lip \phi_j \le 2 \cdot 2^{L^2} + L^2$ and $\phi_j(i) =b_{i,j}$ for $i=0,1,\dots,W^2L^2-1$ and $j=1,2,\dots,J$. We define
\[
\tilde{\phi}(t):= \sum_{j=1}^J 2^{-j} \phi_j(t), \quad t\in \bR.
\]
Then, for $i=0,1,\dots,W^2L^2-1$,
\[
|\tilde{\phi}(i) - \xi_i| = \left| \sum_{j=1}^J 2^{-j}b_{i,j} - \xi_i \right|= | \Bin 0.b_{i,1}b_{i,2} \dots b_{i,J} - \xi_i| \le 2^{-J} \le (WL)^{-2s}.
\]
Since $J\le 1+ 2s\log_2(WL) \le 2(1+s\log_2 W)(1+\log_2 L) \le 2s\log_2(2W)\log_2(2L)$, $\tilde{\phi}$ can be implemented to be a network with width $8s(2W+1)\lceil \log_2 (2W)\rceil +2$ and depth $4L \lceil \log_2(2L) \rceil$, where we use two neurons in each hidden layer to remember the input and intermediate summation. Furthermore, for any $t,t'\in \bR$,
\[
|\tilde{\phi}(t) - \tilde{\phi}(t')| \le \sum_{j=1}^J 2^{-j} \Lip \phi_j|t-t'|\le (4 \cdot 2^{L^2} + 2L^2)|t-t'|.
\]

Finally, we define
\[
\phi(t) := \min\{ \max\{\tilde{\phi}(t),0\},1 \} = \sigma(\tilde{\phi}(t)) - \sigma(\tilde{\phi}(t)-1) \in [0,1].
\]
Then $\phi\in \cN\cN(8s(2W+1)\lceil \log (2W)\rceil,4L \lceil \log(2L)\rceil+1)$, $\Lip \phi \le \Lip \tilde{\phi}$ and $\phi(i)=\tilde{\phi}(i)$ for $i=0,1,\dots,W^2L^2-1$.
\end{proof}

\subsubsection{Approximation of Polynomials}

The approximation of polynomials by ReLU neural networks is well-known \citep{yarotsky2017error,lu2021deep}. The next lemma gives an estimate of the approximation error of the product function.

\begin{lemma}\label{product app}
For any $W,L\in \bN$, there exists $\phi \in \cN\cN(9W+1,L)$ such that for any $x,x',y,y'\in [-1,1]$,
\begin{align*}
|xy - \phi(x,y)| &\le 6W^{-L}, \\
|\phi(x,y) - \phi(x',y')| &\le 7|x-x'| + 7|y-y'|.
\end{align*}
\end{lemma}
\begin{proof}
We follow the construction in \citet{lu2021deep}. We first construct a neural network $\psi$ that approximates the function $f(x)=x^2$ on $[0,1]$. Denote
\[
T_1(x) :=
\begin{cases}
2x, \quad &x\in [0,1/2], \\
2(1-x), \quad &x\in (1/2,1],
\end{cases}
\]
and $T_i(x) := T_{i-1}(T_1(x))$ for $x\in [0,1]$ and $i=2,3,\cdots$. We note that $T_i$ can be implemented by a one-hidden-layer ReLU network with width $2^i$. Let $f_k :[0,1] \to [0,1]$ be the piece-wise linear function such that $f_k(\frac{j}{2^k}) = \left( \frac{j}{2^k}\right)^2$ for $j=0,1,\dots,2^k$, and $f_k$ is linear on $[\frac{j-1}{2^k},\frac{j}{2^k}]$ for $j=1,2,\dots,2^k$. Then, using the fact $\frac{(x-h)^2+(x+h)^2}{2} - x^2 = h^2$, we have
\[
|x^2 - f_k(x)| \le 2^{-2(k+1)}, \quad x\in [0,1], k\in \bN.
\]
Furthermore, $f_{k-1}(x) - f_k(x) = \frac{T_k(x)}{2^{2k}}$ and $x-f_1(x)=\frac{T_1(x)}{4}$. Hence,
\[
f_k(x) = x - (x-f_1(x)) - \sum_{i=2}^k (f_{i-1}(x) - f_i(x)) = x- \sum_{i=1}^k \frac{T_i(x)}{2^{2i}}, \quad x\in[0,1], k\in \bN.
\]

Given $W\in \bN$, there exists a unique $n\in \bN$ such that $(n-1)2^{n-1}+1 \le W\le n2^n$. For any $L\in \bN$, it was showed in \citet[Lemma 5.1]{lu2021deep} that $f_{nL}$ can be implemented by a network $\psi$ with width $3W$ and depth $L$. Hence,
\[
|x^2 -\psi(x)| \le |x^2 - f_{nL}(x)| \le 2^{-2(nL+1)} = 2^{-2nL}/4 \le W^{-L}/4, \quad x\in [0,1],
\]
where we use $W\le n2^n\le 2^{2n}$ in the last inequality.

Using the fact that
\[
xy = 2\left( \left(\tfrac{x+y}{2} \right)^2 - \left( \tfrac{x}{2} \right)^2 - \left( \tfrac{y}{2} \right)^2 \right), \quad x,y\in \bR,
\]
we can approximate the function $f(x,y) = xy$ by
\[
\phi_0(x,y) := 2\left( \psi \left(\tfrac{x+y}{2} \right) - \psi\left( \tfrac{x}{2} \right) - \psi\left( \tfrac{y}{2} \right) \right).
\]
Then, $\phi_0 \in \cN\cN(9W,L)$ and for $x,y\in [0,1]$,
\[
|xy - \phi_0(x,y)| \le 2 \left| \left(\tfrac{x+y}{2} \right)^2 - \psi \left(\tfrac{x+y}{2} \right) \right| + 2 \left| \left(\tfrac{x}{2} \right)^2 - \psi \left(\tfrac{x}{2} \right) \right| + 2 \left| \left(\tfrac{y}{2} \right)^2 - \psi \left(\tfrac{y}{2} \right) \right| \le \tfrac{3}{2}W^{-L}
\]
Furthermore, for any $x,x',y,y'\in [0,1]$,
\begin{align*}
& |\phi_0(x,y) - \phi_0(x',y')| \\
\le& 2 \left|f_{nL} \left(\tfrac{x+y}{2} \right) - f_{nL} \left(\tfrac{x'+y'}{2} \right) \right| + 2 \left|f_{nL} \left(\tfrac{x}{2} \right) - f_{nL} \left(\tfrac{x'}{2} \right) \right| + 2 \left| f_{nL} \left(\tfrac{y}{2} \right) - f_{nL} \left(\tfrac{y'}{2} \right) \right| \\
\le & 4 \left| \tfrac{x+y}{2} - \tfrac{x'+y'}{2} \right| + 2 \left| \tfrac{x}{2} - \tfrac{x'}{2} \right| + 2\left| \tfrac{y}{2} - \tfrac{y'}{2} \right| \\
\le & 3|x-x'| + 3|y-y'|,
\end{align*}
where we use $|f_{nL}(t) - f_{nL}(t')|\le 2|t-t'|$ for $t,t'\in[0,1]$ and $|f_{nL}(t) - f_{nL}(t')|\le |t-t'|$ for $t,t'\in[0,1/2]$.

For any $x,y\in [-1,1]$, set $x_0=(x+1)/2 \in [0,1]$ and $y_0=(y+1)/2\in [0,1]$, then $xy=4x_0y_0-x-y-1$. Using this fact, we define the target function by
\[
\phi(x,y) = 4\phi_0(\tfrac{x+1}{2},\tfrac{y+1}{2}) - \sigma(x+y+2) +1.
\]
Then, $\phi \in \cN\cN(9W+1,L)$ and for $x,y\in [-1,1]$,
\[
|xy - \phi(x,y)|\le 4 |\tfrac{x+1}{2} \tfrac{y+1}{2} - \phi_0(\tfrac{x+1}{2},\tfrac{y+1}{2})|\le 6 W^{-L}.
\]
Furthermore, for any $x,x',y,y'\in [-1,1]$,
\begin{align*}
|\phi(x,y) - \phi(x',y')| \le & 4|\phi_0(\tfrac{x+1}{2},\tfrac{y+1}{2}) - \phi_0(\tfrac{x'+1}{2},\tfrac{y'+1}{2})| + |x+y-x'-y'| \\
\le& 7|x-x'| + 7|y-y'|,
\end{align*}
which completes the proof.
\end{proof}

By applying the approximation of the product function, we can approximate any monomials by neural networks.

\begin{corollary}\label{poly app}
Let $P(x) = x^\alpha = x_1^{\alpha_1} x_2^{\alpha_2}\cdots x_d^{\alpha_d}$ for $\alpha=(\alpha_1,\alpha_2,\dots,\alpha_d)\in \bN_0^d$ with $\|\alpha\|_1 =k\ge 2$. For any $W,L\in \bN$, there exists $\phi \in \cN\cN(9W+k-1,(k-1)(L+1))$ such that for any $x,y\in [-1,1]^d$, $\phi(x) \in [-1,1]$ and
\begin{align*}
|\phi(x) -P(x) |&\le 6(k-1)W^{-L}, \\
|\phi(x) - \phi(y)| &\le 7^{k-1}\|\alpha\|_\infty \|x-y\|_1.
\end{align*}
\end{corollary}
\begin{proof}
For any $x=(x_1,x_2,\dots,x_d) \in \bR^d$, let $z=(z_1,z_2,\dots,z_k)\in \bR^k$ be the vector such that $z_i = x_j$ if $\sum_{l=1}^{j-1} \alpha_l < i \le \sum_{l=1}^j \alpha_l$ for $j=1,2,\dots,d$. Then $P(x)=x^\alpha = z_1z_2 \cdots z_k$ and there exists a linear map $\phi_0:\bR^d \to \bR^k$ such that $\phi_0(x) =z$.

Let $\psi_1 \in \cN\cN(9W+1,L)$ be the neural network in Lemma \ref{product app}. We define
\[
\psi_2(x,y) := \min\{ \max\{\psi_1(x,y),-1\},1 \} = \sigma(\psi_1(x,y)+1) - \sigma(\psi_1(x,y)-1) -1 \in [-1,1],
\]
then $\psi_2 \in \cN\cN(9W+1,L+1)$ and $\psi_2$ also satisfies the inequalities in Lemma \ref{product app}. For $i=3,4,\dots,k$, we define $\psi_i: [-1,1]^i \to [-1,1]$ inductively by
\[
\psi_i(z_1,\dots,z_i) := \psi_2(\psi_{i-1}(z_1,\dots,z_{i-1}), z_i).
\]
Since $z_i=\sigma(z_i+1)-1$ for $z_i\in [-1,1]$, it is easy to see that $\psi_i$ can be implemented by a network with width $9W+i-1$ and depth $(i-1)(L+1)$ by induction. Furthermore,
\begin{align*}
&|\psi_i(z_1,\dots,z_i) - z_1\cdots z_i| \\
\le & |\psi_2(\psi_{i-1}(z_1,\dots,z_{i-1}), z_i) - \psi_{i-1}(z_1,\dots,z_{i-1})z_i| + |\psi_{i-1}(z_1,\dots,z_{i-1})z_i - z_1\cdots z_i| \\
\le & 6W^{-L} + |\psi_{i-1}(z_1,\dots,z_{i-1}) - z_1\cdots z_{i-1}| \\
\le & \cdots \le (i-2)6W^{-L} + |\psi_2(z_1,z_2) - z_1z_2| \\
\le & (i-1)6W^{-L}.
\end{align*}
And for any $z=(z_1,z_2,\dots,z_k), z'= (z_1',z_2',\dots,z_k') \in [-1,1]^k$,
\begin{align*}
|\psi_i(z_1,\dots,z_i) - \psi_i(z_1',\dots,z_i')| &\le 7|\psi_{i-1}(z_1,\dots,z_{i-1}) - \psi_{i-1}(z_1',\dots,z_{i-1}')| + 7|z_i-z_i'| \\
&\le \cdots \le 7^{i-2} |\psi_2(z_1,z_2) -\psi_2(z_1',z_2')| + \sum_{j=3}^i 7^{i-j+1}|z_j-z_j'| \\
&\le 7^{i-1} \|z-z'\|_1.
\end{align*}

We define the target function as $\phi(x) := \psi_k(\phi_0(x))$, then $\phi \in \cN\cN(9W+k-1,(k-1)(L+1))$. And for $x,y\in [-1,1]^d$, denote $z=\phi_0(x)$ and $z'=\phi_0(y)$, we have
\begin{align*}
|\phi(x) -P(x) | &= |\psi_k(z) - z_1z_2\cdots z_k|\le 6(k-1)W^{-L}, \\
|\phi(x) - \phi(y)| &= |\psi_k(z) - \psi_k(z')|\le 7^{k-1}\|z-z'\|_1 \le 7^{k-1}\|\alpha\|_\infty \|x-y\|_1.
\end{align*}
So we finish the proof.
\end{proof}

\subsubsection{Proof of Lemma \ref{discriminator approximation}} \label{subsec: proof of dis app}

Now, we can bound the discriminator approximation error. We recall Lemma \ref{discriminator approximation} in the following and give a proof.

\disapp*


\begin{proof}
We divide the proof into four steps as follows.

\noindent \textbf{Step 1}: Discretization.

Let $K=\lfloor (WL)^{2/d}\rfloor$ and $\delta = \tfrac{1}{3K^{\beta \lor 1}}\le \tfrac{1}{3K}$. For each $\theta=(\theta_1,\theta_2,\dots,\theta_d)\in \{0,1,\dots,K-1 \}^d$, we define
\[
Q_\theta := \left\{x=(x_1,x_2,\dots,x_d): x_i\in \left[ \tfrac{\theta_i}{K}, \tfrac{\theta_i+1}{K} - \delta\cdot 1_{\{\theta_i<K-1\}} \right], i=1,2,\dots,d \right\}.
\]
By Proposition \ref{partition map}, there exists $\psi_1\in \cN\cN(4W+3,4L)$ such that
\[
\psi_1(t) = \tfrac{k}{K}, \quad \mbox{if } t\in \left[\tfrac{k}{K}, \tfrac{k+1}{K}-\delta\cdot 1_{\{k<K-1\}} \right], k=0,1,\dots,K-1,
\]
and $\Lip \psi_1\le 2LK^{-2}\delta^{-2}$. We define
\[
\psi(x) := (\psi_1(x_1),\dots,\psi_1(x_d)), \quad x=(x_1,\dots,x_d) \in \bR^d.
\]
Then, $\psi\in \cN\cN(d(4W+3),4L)$ and $\psi(x) = \tfrac{\theta}{K}$ for $x\in Q_\theta$.

\noindent \textbf{Step 2}: Approximation of Taylor coefficients.

Since $\theta \in \{0,1,\dots,K-1\}^d$ is one-to-one correspondence to $i_\theta := \sum_{j=1}^d \theta_j K^{j-1} \in \{0,1,\dots,K^d-1\}$, we define
\[
\psi_0(x):= (K,K^2,\dots,K^d) \cdot \psi(x) = \sum_{j=1}^d \psi_1(x_j) K^j \quad x\in \bR^d,
\]
then $\psi_0\in \cN\cN(d(4W+3),4L)$ and
\[
\psi_0(x) = \sum_{j=1}^d \theta_j K^{j-1}=i_\theta \quad \mbox{if } x\in Q_\theta,\ \theta \in \{0,1,\dots,K-1\}^d.
\]
For any $x,x'\in \bR^d$, we have
\[
|\psi_0(x)-\psi_0(x')| \le \sum_{j=1}^d K^j |\psi_1(x_j)-\psi_1(x_j')| \le \sqrt{d}K^d \Lip \psi_1 \|x-x'\|_2 \le 2\sqrt{d} LK^{d-2}\delta^{-2} \|x-x'\|_2.
\]

For any $\alpha\in \bN_0^d$ satisfying $\|\alpha\|_1\le s$ and each $i=i_\theta \in \{0,1,\dots,K^d-1\}$, we denote $\xi_{\alpha,i} := (\partial^\alpha h(\theta /K)+1)/2 \in [0,1]$. Since $K^d \le W^2L^2$, by Proposition \ref{data fitting}, there exists $\varphi_\alpha\in \cN\cN(8(s+1)(2W+1) \lceil \log_2 (2W)\rceil+2,4L \lceil \log_2 (2L)\rceil+1)$ such that $\Lip \varphi_\alpha \le 4 \cdot 2^{L^2} + 2L^2 \le 5\cdot 2^{L^2}$ and $|\varphi_\alpha(i) - \xi_{\alpha,i}|\le (WL)^{-2(s+1)}$ for all $i \in \{0,1,\dots,K^d-1\}$. We define
\[
\phi_\alpha(x) := 2\varphi_\alpha(\psi_0(x))-1 \in [-1,1], \quad x\in \bR^d.
\]
Then $\phi_\alpha$ can be implemented by a network with width $8d(s+1)(2W+1)\lceil \log_2 (2W)\rceil +2\le 40d(s+1)W \lceil \log_2 W\rceil$ and depth $4L +4L\lceil \log_2 (2L)\rceil+1 \le 13L \lceil\log_2 L \rceil$. And we have
\begin{equation}\label{phi_alpha lip}
\Lip \phi_\alpha \le 2\Lip \varphi_\alpha \Lip \psi_0\le 20\sqrt{d} LK^{d-2}\delta^{-2} 2^{L^2},
\end{equation}
and for any $\theta\in \{0,1,\dots,K-1 \}^d$, if $x\in Q_\theta$,
\begin{equation}\label{phi_alpha bound}
|\phi_\alpha(x)-\partial^\alpha h(\theta/K)| = 2|\varphi_\alpha(i_\theta)-\xi_{\alpha,i_\theta}|\le 2(WL)^{-2(s+1)}.
\end{equation}

\noindent \textbf{Step 3}: Approximation of $h$ on $\bigcup_{\theta\in \{0,1,\dots,K-1 \}^d} Q_\theta$.

Let $\varphi(t) = \min\{\max\{t,0\},1 \} = \sigma(t) - \sigma(t-1)$ for $t\in \bR$. We extend its definition to $\bR^d$ coordinate-wisely, so $\varphi:\bR^d \to [0,1]^d$ and $\varphi(x) = x$ for any $x\in [0,1]^d$.

By Lemma \ref{product app}, there exists $\phi_\times \in \cN\cN(9W+1,2(s+1)L)$ such that for any $t_1,t_2,t_3,t_4\in [-1,1]$,
\begin{align}
|t_1t_2 - \phi_\times(t_1,t_2)| &\le  6W^{-2(s+1)L}, \label{phi_times bound} \\
|\phi_\times(t_1,t_2) - \phi_\times(t_3,t_4)| &\le  7|t_1-t_3| + 7|t_2-t_4|. \label{phi_times lip}
\end{align}
By corollary \ref{poly app}, for any $\alpha\in \bN_0^d$ with $2\le \|\alpha\|_1\le s$,  there exists $P_\alpha\in \cN\cN(9W+s-1,(s-1)(2(s+1)L+1))$ such that for any $x,y\in [-1,1]^d$, $P_\alpha(x) \in [-1,1]$ and
\begin{align}
|P_\alpha(x) -x^\alpha |\le 6 (s-1)W^{-2(s+1)L}, \label{P_alpha bound} \\
|P_\alpha(x) - P_\alpha(y)| \le 7^{s-1}s \|x-y\|_1. \label{P_alpha lip}
\end{align}
When $\|\alpha\|_1= 1$, it is easy to implemented $P_\alpha(x) = x^\alpha$ by a neural network with Lipschitz constant at most one. Hence, the inequalities (\ref{P_alpha bound}) and (\ref{P_alpha lip}) hold for $1\le \|\alpha\|_1\le s$.

For any $x\in Q_\theta$, $\theta\in \{0,1,\dots,K-1\}^d$, we can approximate $h(x)$ by a Taylor expansion. Thanks to \citet[Lemma A.8]{petersen2018optimal}, we have the following error estimation for $x\in Q_\theta$,
\begin{equation}\label{Taylor app bound}
\left| h(x) - h(\tfrac{\theta}{K}) - \sum_{1\le \|\alpha\|_1 \le s} \frac{\partial^\alpha h(\tfrac{\theta}{K})}{\alpha !} (x-\tfrac{\theta}{K})^\alpha \right| \le d^s \|x-\tfrac{\theta}{K}\|_2^\beta \le d^{s+\beta/2} K^{-\beta}.
\end{equation}
Motivated by this, we define
\begin{align*}
\widetilde{\phi}_0(x) &:= \phi_{\mathbf{0}_d}(x) + \sum_{1\le \|\alpha\|_1 \le s} \phi_\times \left( \tfrac{\phi_\alpha(x)}{\alpha !}, P_\alpha(\varphi(x) - \psi(x) ) \right),\\
\phi_0(x) &:= \sigma(\widetilde{\phi}_0(x)+1) - \sigma(\widetilde{\phi}_0(x)-1) -1 \in [-1,1],
\end{align*}
where we denote $\mathbf{0}_d =(0,\dots,0) \in \bN_0^d$. Observe that the number of terms in the summation can be bounded by
\[
\sum_{\alpha \in \bN_0^d, \|\alpha\|_1\le s} 1 = \sum_{j=0}^s \sum_{\alpha \in \bN_0^d, \|\alpha\|_1= j} 1 \le \sum_{j=0}^s d^j \le (s+1)d^s.
\]
Recall that $\varphi\in \cN\cN(2d,1)$, $\psi\in \cN\cN(d(4W+3),4L)$, $P_\alpha\in \cN\cN(9W+s-1,2(s^2-1)L+s-1)$, $\phi_\alpha\in \cN\cN(40d(s+1)W \lceil\log_2 W\rceil, 13L \lceil\log_2 L\rceil)$ and $\phi_\times \in \cN\cN(9W+1,2(s+1)L)$. Hence, by our construction, $\phi_0$ can be implemented by a neural network with width $49(s+1)^2 d^{s+1}W \lceil\log_2 W\rceil$ and depth $15(s+1)^2 L \lceil\log_2 L\rceil$.

For any $1\le \|\alpha\|_1\le s$ and $x,y\in \bR^d$, since $\phi_\alpha(x), \phi_\alpha(y), \varphi(x) - \psi(x), \varphi(y) - \psi(y) \in [-1,1]$, by inequalities (\ref{phi_alpha lip}), (\ref{phi_times lip}) and (\ref{P_alpha lip}), we have
\begin{align*}
&\left| \phi_\times \left( \tfrac{\phi_\alpha(x)}{\alpha !}, P_\alpha(\varphi(x) - \psi(x) ) \right) - \phi_\times \left( \tfrac{\phi_\alpha(y)}{\alpha !}, P_\alpha(\varphi(y) - \psi(y) ) \right) \right| \\
\le & 7|\phi_\alpha(x) - \phi_\alpha(y)| + 7 |P_\alpha(\varphi(x) - \psi(x)) - P_\alpha(\varphi(y) - \psi(y)) | \\
\le & 7 \Lip \phi_\alpha \|x-y\|_2 + s7^s \| \varphi(x) - \varphi(y)\|_1 + s7^s \|\psi(x) - \psi(y)\|_1 \\
\le &  140\sqrt{d} LK^{d-2}\delta^{-2} 2^{L^2}\|x-y\|_2 + s7^s \sqrt{d} \| x-y\|_2 + 2s 7^s \sqrt{d}LK^{-2}\delta^{-2} \|x-y\|_2 \\
\le & \sqrt{d}LK^{2(\beta \lor 1)-2} (1260K^d 2^{L^2}+ 19s 7^s) \|x-y\|_2.
\end{align*}
One can check that the bound also holds for $\|\alpha\|_1=0$ and $s=0$. Hence,
\begin{align*}
\Lip\phi_0 \le \Lip \widetilde{\phi}_0&\le \sum_{\|\alpha\|_1\le s} \sqrt{d}LK^{2(\beta \lor 1)-2} (1260K^d 2^{L^2}+ 19s 7^s) \\
&\le (s+1) d^{s+1/2}L(WL)^{\sigma(4\beta-4)/d} (1260 W^2L^2 2^{L^2}+ 19s 7^s).
\end{align*}

We can estimate the error $|h(x) - \phi_0(x)|$ as follows. For any $x\in Q_\theta$, we have $\varphi(x)=x$ and $\psi(x) =\tfrac{\theta}{K}$. Hence, by the triangle inequality and inequality (\ref{Taylor app bound}),
\begin{align*}
&|h(x) - \phi_0(x)| \le |h(x) - \widetilde{\phi}_0(x)| \\
\le& |h(\tfrac{\theta}{K}) - \phi_{\mathbf{0}_d}(x)| + \sum_{1\le \|\alpha\|_1 \le s} \left|  \frac{\partial^\alpha h(\tfrac{\theta}{K})}{\alpha !} (x-\tfrac{\theta}{K})^\alpha -  \phi_\times \left( \tfrac{\phi_\alpha(x)}{\alpha !}, P_\alpha(x- \tfrac{\theta}{K} ) \right) \right| + d^{s+\beta/2} K^{-\beta} \\
=&: \sum_{\|\alpha\|_1 \le s} \cE_\alpha + d^{s+\beta/2} \lfloor (WL)^{2/d}\rfloor^{-\beta}.
\end{align*}
Using the inequality $|t_1t_2 - \phi_\times(t_3,t_4)| \le |t_1t_2 - t_3t_2| + |t_3t_2 - t_3t_4| + |t_3t_4 - \phi_\times(t_3,t_4)| \le |t_1 - t_3| + |t_2 - t_4| + |t_3t_4 - \phi_\times(t_3,t_4)|$ for any $t_1,t_2,t_3,t_4\in [-1,1]$ and the inequalities (\ref{phi_alpha bound}), (\ref{phi_times bound}) and (\ref{P_alpha bound}), we have for $1\le \|\alpha\|_1\le s$,
\begin{align*}
\cE_\alpha \le& \tfrac{1}{\alpha!} \left|\partial^\alpha h(\tfrac{\theta}{K}) - \phi_\alpha(x)\right| + \left|(x-\tfrac{\theta}{K})^\alpha - P_\alpha(x-\tfrac{\theta}{K})\right| \\
& \quad + \left|\tfrac{\phi_\alpha(x)}{\alpha !}P_\alpha(x- \tfrac{\theta}{K} ) - \phi_\times \left( \tfrac{\phi_\alpha(x)}{\alpha !}, P_\alpha(x- \tfrac{\theta}{K} ) \right) \right| \\
\le & 2(WL)^{-2(s+1)} + 6(s-1)W^{-2(s+1)L} + 6W^{-2(s+1)L} \\
\le & (6s+2)(WL)^{-2(s+1)}.
\end{align*}
It is easy to check that the bound is also true for $\|\alpha\|_1=0$ and $s=0$. Therefore,
\begin{align*}
|h(x) - \phi_0(x)| &\le \sum_{\|\alpha\|_1 \le s} (6s+2)(WL)^{-2(s+1)} + d^{s+\beta/2} \lfloor (WL)^{2/d}\rfloor^{-\beta} \\
&\le (s+1)d^s(6s+2)(WL)^{-2(s+1)} + d^{s+\beta/2} \lfloor (WL)^{2/d}\rfloor^{-\beta} \\
&\le (6s+3)(s+1)d^{s+\beta/2} \lfloor (WL)^{2/d}\rfloor^{-\beta} \\
&=: \cE,
\end{align*}
for any  $x\in \bigcup_{\theta\in \{0,1,\dots,K-1 \}^d} Q_\theta$.

\noindent \textbf{Step 4}: Approximation of $h$ on $[0,1]^d$.

Next, we construct a neural network $\phi$ that uniformly approximates $h$ on $[0,1]^d$. To present the construction, we denote $\mid(t_1,t_2,t_3)$ as the function that returns the middle value of three inputs $t_1,t_2,t_3 \in \bR$. It is easy to check that
\[
\max\{t_1,t_2 \} = \frac{1}{2} (\sigma(t_1+t_2) - \sigma(-t_1-t_2) + \sigma(t_1-t_2) + \sigma(t_2-t_1) )
\]
Thus, $\max\{t_1,t_2,t_3\} = \max\{\max\{t_1,t_2\},\sigma(t_3)-\sigma(-t_3) \}$ can be implemented by a network with width $6$ and depth $2$. Similar construction holds for $\min\{t_1,t_2,t_3\}$. Since
\[
\mid(t_1,t_2,t_3) = \sigma(t_1+t_2+t_3) - \sigma(-t_1-t_2-t_3) - \max\{t_1,t_2,t_3 \} - \min\{t_1,t_2,t_3 \},
\]
it is easy to see $\mid(\cdot,\cdot,\cdot) \in \cN\cN(14,2)$.

Recall that $\phi_0\in \cN\cN(49(s+1)^2 d^{s+1}W \lceil\log_2 W\rceil,15(s+1)^2 L \lceil\log_2 L\rceil)$. Let $\{e_i\}_{i=1}^d$ be the standard basis in $\bR^d$. We inductively define
\[
\phi_i(x) := \mid (\phi_{i-1}(x-\delta e_i), \phi_{i-1}(x), \phi_{i-1}(x+\delta e_i) ) \in [-1,1], \quad i=1,2,\dots,d.
\]
Then $\phi_d \in \cN\cN(49(s+1)^2 3^d d^{s+1}W \lceil\log_2 W\rceil, 15(s+1)^2 L \lceil \log_2 L\rceil+2d)$. For any $x,x'\in \bR^d$, the functions $\phi_{i-1}(\cdot-\delta e_i)$, $\phi_{i-1}(\cdot)$ and $\phi_{i-1}(\cdot+\delta e_i)$ are piece-wise linear on the segment that connecting $x$ and $x'$. Hence, the Lipschitz constant of these functions on the segment is the maximum absolute value of the slopes of linear parts. Since the middle function does not increase the maximum absolute value of the slopes, it does not increase the Lipschitz constant, which shows that $\Lip \phi_d \le \Lip \phi_0$.

Denote $Q(K,\delta) := \bigcup_{k=0}^{K-1} [\frac{k}{K}, \frac{k+1}{K}-\delta \cdot 1_{\{k<K-1\}} ]$ and define, for $i=0,1,\dots,d$,
\[
E_i := \{ (x_1,x_2,\dots,x_d)\in [0,1]^d: x_j\in Q(K,\delta), j>i \},
\]
then $E_0 = \bigcup_{\theta\in \{0,1,\dots,K-1 \}^d} Q_\theta$ and $E_d = [0,1]^d$. We assert that
\[
|\phi_i(x) - h(x)|\le \cE + i\delta^{\beta \land 1}, \quad \forall x\in E_i, i=0,1,\dots,d.
\]

We prove the assertion by induction. By construction, it is true for $i=0$. Assume the assertion is true for some $i$, we will prove that it is also holds for $i+1$. For any $x\in E_{i+1}$, at least two of $x-\delta e_{i+1}$, $x$ and $x+\delta e_{i+1}$ are in $E_i$. Therefore, by assumption and the inequality $|h(x)-h(x\pm \delta e_{i+1})|\le \delta^{\beta \land 1}$, at least two of the following inequalities hold
\begin{align*}
|\phi_i(x-\delta e_{i+1}) - h(x)| &\le |\phi_i(x-\delta e_{i+1}) - h(x-\delta e_{i+1})| + \delta^{\beta \land 1} \le \cE + (i+1)\delta^{\beta \land 1},\\
|\phi_i(x) - h(x)| &\le \cE + i\delta^{\beta \land 1}, \\
|\phi_i(x+\delta e_{i+1}) - h(x)| &\le |\phi_i(x+\delta e_{i+1}) - h(x+\delta e_{i+1})| + \delta^{\beta \land 1} \le \cE + (i+1)\delta^{\beta \land 1}.
\end{align*}
In other words, at least two of $\phi_i(x-\delta e_{i+1})$, $\phi_i(x)$ and $\phi_i(x+\delta e_{i+1})$ are in the interval $[h(x)-\cE - (i+1)\delta^{\beta \land 1}, h(x)+\cE + (i+1)\delta^{\beta \land 1} ]$. Hence, their middle value $\phi_{i+1}(x) = \mid(\phi_i(x-\delta e_{i+1}), \phi_i(x), \phi_i(x+\delta e_{i+1}))$ must be in the same interval, which means
\[
|\phi_{i+1}(x) -h(x)| \le \cE + (i+1)\delta^{\beta \land 1}.
\]
So the assertion is true for $i+1$.

Recall that
\[
\delta^{\beta \land 1} = \left( \frac{1}{3K^{\beta \lor 1}} \right)^{\beta \land 1} =
\begin{cases}
\frac{1}{3} K^{-\beta} \quad &\beta \ge 1, \\
(3K)^{-\beta} \quad &\beta < 1,
\end{cases}
\]
and $K=\lfloor (WL)^{2/d}\rfloor$. Since $E_d=[0,1]^d$, let $\phi := \phi_d$, we have
\begin{align*}
\| \phi - h\|_{L^\infty([0,1]^d)} &\le \cE + d\delta^{\beta \land 1} \\
&\le (6s+3)(s+1)d^{s+\beta/2} \lfloor (WL)^{2/d}\rfloor^{-\beta} + d \lfloor (WL)^{2/d}\rfloor^{-\beta} \\
&\le 6(s+1)^2 d^{(s+\beta/2) \lor 1} \lfloor (WL)^{2/d}\rfloor^{-\beta},
\end{align*}
which completes the proof.
\end{proof}

\subsection{Bounding Statistical Error}\label{sec: stat err}

The technique for bounding the statistical error is rather standard \citep{anthony2009neural,shalevshwartz2014understanding,mohri2018foundations}. We first show that the statistical error $\bE [d_\cF(\mu,\widehat{\mu}_n)]$ of a function class $\cF$ can be bounded by the Rademacher complexity, and then bound the Rademacher complexity by Dudley's entropy integral \citep{dudley1967sizes}. We restate Lemma \ref{statistical error bound} here for convenience.

\staterr*


\begin{proof}
Recall that we have $n$ i.i.d. samples $X_{1:n}:=\{X_i \}_{i=1}^n$ from $\mu$ and
$\widehat{\mu}_n = \frac{1}{n} \sum_{i=1}^{n} \delta_{X_i}$. We introduce a ghost data set $X'_{1:n}=\{X_i'\}_{i=1}^n$ drawn i.i.d. from $\mu$, then
\begin{align*}
\bE_{X_{1:n}} [d_\cF(\mu,\widehat{\mu}_n)] &= \bE_{X_{1:n}} \left[\sup_{f\in \cF} \bE_{x\sim \mu}[f(x)] - \frac{1}{n} \sum_{i=1}^n f(X_i) \right] \\
&= \bE_{X_{1:n}} \left[\sup_{f\in \cF} \bE_{X'_{1:n}}  \frac{1}{n} \sum_{i=1}^n f(X'_i) - \frac{1}{n} \sum_{i=1}^n f(X_i) \right] \\
&\le \bE_{X_{1:n},X'_{1:n}} \left[\sup_{f\in \cF} \frac{1}{n} \sum_{i=1}^n (f(X'_i) -  f(X_i)) \right].
\end{align*}
Let $\xi = \{\xi_i\}_{i=1}^n$ be a sequence of i.i.d. Rademacher variables independent of $X_{1:n}$ and $X'_{1:n}$. Then, by symmetrization, we can bound $\bE_{X_{1:n}} [d_\cF(\mu,\widehat{\mu}_n)]$ by the Rademacher complexity of $\cF$:
\begin{align*}
\bE_{X_{1:n}} [d_\cF(\mu,\widehat{\mu}_n)] &\le \bE_{X_{1:n},X'_{1:n}} \left[\sup_{f\in \cF} \frac{1}{n} \sum_{i=1}^n (f(X'_i) -  f(X_i)) \right] \\
&= \bE_{X_{1:n},X'_{1:n},\xi} \left[\sup_{f\in \cF} \frac{1}{n} \sum_{i=1}^n \xi_i (f(X'_i) -  f(X_i)) \right] \\
&\le \bE_{X_{1:n},X'_{1:n},\xi} \left[\sup_{f\in \cF} \frac{1}{n} \sum_{i=1}^n \xi_i f(X'_i) + \sup_{f\in \cF} \frac{1}{n} \sum_{i=1}^n -\xi_i f(X_i) \right] \\
&= 2 \bE_{X_{1:n},\xi} \left[\sup_{f\in \cF} \frac{1}{n} \sum_{i=1}^n \xi_i f(X_i) \right],
\end{align*}
where the last equality is due to the fact that $X_i$ and $X'_i$ have the same distribution and the fact that $\xi_i$ and $-\xi_i$ have the same distribution.

For any $A \subseteq \bR^n$, we denote the Rademacher complexity of $A$ by
\[
\cR(A) :=  \bE_{\xi} \left[ \sup_{(a_1,\dots,a_n)\in A} \frac{1}{n} \sum_{i=1}^n \xi_i a_i \right].
\]
Then, if we denote $\cF_{|_{X_{1:n}}} = \{(f(X_1),\dots,f(X_N)):f\in \cF \}$ for any fixed $X_{1:n}=\{X_i \}_{i=1}^n$, we have shown
\[
\bE_{X_{1:n}} [d_\cF(\mu,\widehat{\mu}_n)] \le 2 \bE_{X_{1:n}} [\cR(\cF_{|_{X_{1:n}}})].
\]
We define a distance of two vectors $x,y\in\bR^n$ by
\[
d_2(x,y) := \left(\frac{1}{n}\sum_{i=1}^n(x_i-y_i)^2 \right)^{1/2} = \frac{1}{\sqrt{n}} \|x-y\|_2.
\]
The corresponding $\epsilon$-covering number of the set $\cF_{|_{X_{1:n}}} \subseteq \bR^n$ is denoted by $\cN(\epsilon,\cF_{|_{X_{1:n}}},d_2)$. By chaining technique (see \citet[Lemma 27.4]{shalevshwartz2014understanding}), one can show that for any integer $K\ge0$,
\begin{align*}
\cR(\cF_{|_{X_{1:n}}}) &\le 2^{-K}B + \frac{6B}{\sqrt{n}} \sum_{k=1}^K 2^{-k} \sqrt{\log \cN(2^{-k}B,\cF_{|_{X_{1:n}}},d_2)} \\
&\le 2^{-K}B + \frac{12}{\sqrt{n}} \int_{2^{-K-1}B}^{B/2} \sqrt{\log \cN(\epsilon,\cF_{|_{X_{1:n}}},d_2)} d\epsilon.
\end{align*}
Now, for any $\delta \in (0,B/2)$, there exists an integer $K$ such that $2^{-K-2}B\le \delta<2^{-K-1}B$. Therefore, we have
\begin{align*}
\bE_{X_{1:n}} [d_\cF(\mu,\widehat{\mu}_n)] &\le 2 \bE_{X_{1:n}} \cR(\cF_{|_{X_{1:n}}}) \\
&\le 2 \bE_{X_{1:n}} \inf_{0< \delta<B/2}\left( 4\delta + \frac{12}{\sqrt{n}} \int_{\delta}^{B/2} \sqrt{\log \cN(\epsilon,\cF_{|_{X_{1:n}}},d_2)} d\epsilon \right).
\end{align*}
Since $d_2(x,y) \le \|x-y\|_\infty$, we have $\cN(\epsilon,\cF_{|_{X_{1:n}}},d_2) \le \cN(\epsilon,\cF_{|_{X_{1:n}}},\|\cdot\|_\infty)$, which completes the proof.
\end{proof}

When the function class $\cF$ has a finite pseudo-dimension, we can further bound the covering number by the pseudo-dimension of $\cF$.

\begin{corollary}\label{statistical error bound by pdim}
Assume $\sup_{f\in \cF}\|f\|_\infty \le B$ and the pseudo-dimension of $\cF$ is $\Pdim(\cF)<\infty$, then
\[
\bE [d_\cF(\mu,\widehat{\mu}_n)] \le CB  \sqrt{\frac{ \Pdim(\cF) \log n}{n}}
\]
for some universal constant $C>0$.
\end{corollary}
\begin{proof}
If $n\ge \Pdim(\cF)$, we have the following bound from \citet[Theorem 12.2]{anthony2009neural},
\[
\cN(\epsilon,\cF_{|_{X_{1:n}}},\|\cdot\|_\infty) \le \left( \frac{2eBn}{\epsilon \Pdim(\cF)} \right)^{\Pdim(\cF)}.
\]
If $n< \Pdim(\cF)$, since $\cF_{|_{X_{1:n}}} \subseteq \{x\in \bR^n:\|x\|_\infty\le B \}$ can be covered by at most $\lceil\frac{2B}{\epsilon} \rceil^n$ balls with radius $\epsilon$ in $\|\cdot\|_\infty$ distance, we always have $\cN(\epsilon,\cF_{|_{X_{1:n}}},\|\cdot\|_\infty)\le \lceil\frac{2B}{\epsilon} \rceil^n$. In any cases,
\[
\log \cN(\epsilon,\cF_{|_{X_{1:n}}},\|\cdot\|_\infty) \le \Pdim(\cF) \log \frac{2eBn}{\epsilon}.
\]
As a consequence,
\begin{align*}
\bE [d_\cF(\mu,\widehat{\mu}_n)] &\le \inf_{0< \delta<B/2}\left( 8\delta + 24 \sqrt{\frac{ \Pdim(\cF)}{n}} \int_{\delta}^{B/2} \sqrt{\log(2eBn/\epsilon) } d\epsilon \right) \\
&\le \inf_{0< \delta<B/2}\left( 8\delta + 12B \sqrt{\frac{ \Pdim(\cF) \log(2eBn/\delta)}{n}} \right) \\
&\le CB  \sqrt{\frac{ \Pdim(\cF) \log n}{n}}
\end{align*}
for some universal constant $C>0$.
\end{proof}

\acks{The work of Y. Jiao is supported in part by the National Science Foundation of China under Grant 11871474 and by the research
fund of KLATASDSMOE. The research of Y. Wang is supported by the HK RGC grant 16308518, the HK Innovation Technology Fund Grant  ITS/044/18FX and the Guangdong-Hong Kong-Macao Joint Laboratory for Data Driven Fluid Dynamics and Engineering Applications (Project 2020B1212030001). We thank the editor and reviewers for their feedback on our manuscript.}

\bibliography{Ref}
\end{document}